\documentclass[twoside,11pt]{article}
\usepackage{journal2e_v1}

\usepackage{natbib}
\usepackage{epsf}
\usepackage{epsfig}
\usepackage{amsmath}
\usepackage{subfigure}
\usepackage{amssymb}
\usepackage{algorithm}
\usepackage{algpseudocode}

\def\A{{\bf A}}
\def\a{{\bf a}}
\def\B{{\bf B}}
\def\b{{\bf b}}
\def\C{{\bf C}}

\def\D{{\bf D}}
\def\d{{\bf d}}
\def\E{{\bf E}}

\def\F{{\bf F}}
\def\f{{\bf f}}
\def\G{{\bf G}}

\def\K{{\bf K}}
\def\H{{\bf H}}
\def\I{{\bf I}}
\def\L{{\bf L}}
\def\M{{\bf M}}

\def\BP{{\bf P}}
\def\R{{\bf R}}
\def\BS{{\bf S}}

\def\T{{\bf T}}
\def\U{{\bf U}}
\def\u{{\bf u}}
\def\V{{\bf V}}

\def\W{{\bf W}}

\def\X{{\bf X}}
\def\Y{{\bf Y}}
\def\Q{{\bf Q}}
\def\x{{\bf x}}
\def\y{{\bf y}}
\def\Z{{\bf Z}}

\def\0{{\bf 0}}
\def\1{{\bf 1}}

\def\EB{{\mathbb E}}
\def\BR{{\mathbb R}}
\def\SR{{\mathbb S}}

\def\ph{\mbox{\boldmath$\phi$\unboldmath}}

\def\Ph{\mbox{\boldmath$\Phi$\unboldmath}}

\def\Ps{\mbox{\boldmath$\Psi$\unboldmath}}

\def\Si{\mbox{\boldmath$\Sigma$\unboldmath}}

\def\Lam{\mbox{\boldmath$\Lambda$\unboldmath}}
\def\Gam{\mbox{\boldmath$\Gamma$\unboldmath}}
\def\Oma{\mbox{\boldmath$\Omega$\unboldmath}}
\def\De{\mbox{\boldmath$\Delta$\unboldmath}}

\def\Tha{\mbox{\boldmath$\Theta$\unboldmath}}

\def\Ups{\mbox{\boldmath$\Upsilon$\unboldmath}}

\def\vect{\mathrm{vec}}
\def\tr{\mathrm{tr}}
\def\rk{\mathrm{rk}}
\def\diag{\mathrm{diag}}

\def\vecd{\mathrm{vec}}
\def\ran{\mathrm{range}}

\begin{document}

\title{The Matrix Ridge Approximation: Algorithms and Applications}

\author{\name Zhihua Zhang  \\
\addr  MOE-Microsoft Key Lab for Intelligent Computing and Intelligent Systems \\
Department of Computer Science and Engineering \\
Shanghai Jiao Tong University \\
800 Dong Chuan Road, Shanghai, China 200240 \\
\texttt{zhihua@sjtu.edu.cn}
}

\date{Revised, November 27, 2013}

\maketitle

\begin{abstract}%
We are concerned with an approximation problem for a symmetric positive semidefinite matrix due to
motivation from a class of nonlinear machine learning methods. We discuss an approximation approach that we call
\emph{matrix ridge approximation}. In particular,
we define the matrix ridge approximation as an
incomplete matrix factorization plus a ridge term. Moreover,
we present
probabilistic interpretations using a normal latent
variable model and a Wishart model for this approximation approach. The idea behind the latent variable model in turn
leads us to an efficient EM iterative method for handling the matrix ridge
approximation problem. Finally, we illustrate the applications of the
approximation approach in multivariate data analysis. Empirical studies in
spectral clustering and Gaussian process regression  show that the matrix ridge
approximation with the EM iteration is potentially useful.
\end{abstract}

\begin{keywords} Positive semidefinite matrices; Matrix ridge approximation;
Incomplete matrix factorization; Expectation maximization algorithms; Probabilistic models.
\end{keywords}

\section{Introduction}

Symmetric  positive semidefinite matrices
play an important role in multivariate statistical analysis and machine
learning. Especially, the low-rank approximation of a positive semidefinite matrix
has been widely applied to multivariate data analysis.
In this paper we study the  low-rank approximation problem of a positive semidefinite matrix as well as its applications in machine learning.
Moreover, we always assume that the positive semidefinite matrix in question is symmetric.

Some machine learning methods require computing the
inverse of a positive definite matrix or the spectral  decomposition of a positive semidefinite matrix.
For example, the kernel PCA (principal component analysis)~\citep{ScholkopfBook:2002}, classical multidimensional scaling
(also called principal coordinate analysis, PCO)~\citep{Mardia:1979} and spectral clustering algorithms~\citep{ZhangzhSTS:2008a} require solving an
eigenvalue problem with linear constraints on an $m{\times} m$ inner-product matrix~\citep{GolubSIAM:1973},
and Gaussian processes (GPs)~\citep{RassmussenWilliams}  need to  invert  $m{\times} m$ covariance
matrices. Typically, these methods take $O(m^3)$ operations where $m$ denotes the number
of training instances. This scaling is unfavorable for applications in massive datasets.

Several approaches have been also proposed to address this
computational challenge, such as randomized techniques \citep{Achlioptas:2001},
sparse greedy approximation~\citep{Smola:2000}, and the Nystr\"{o}m
method \citep{WilliamsNIPS:2001_2,YangNIPS:2012}. All these approaches are
based on sampling techniques. Similar ideas include random Fourier features~\citep{RahimiNIPS:2008,QuinoneroGPR:2007,LazaroJMLR:2010,LeSarlosSmola:2013} and
hashing features~\citep{ShiJMLR:2009}. Specifically, the random feature method avoids inversion of a matrix by solving a linear system of equations instead.  Another widely used approach
is to employ the incomplete Cholesky decomposition method~\citep{Golub:1996,Fine:2001}. The approach is deterministic.
Although these approaches can be efficient, their range of applications
might be limited; e.g., these approaches are always infeasible in handling the
eigenvalue decomposition problem with linear constraints.

In this paper we present a new deterministic low-rank approximation approach.
Roughly speaking,  the approach is to approximate a positive semidefinite matrix as an
incomplete matrix  decomposition  plus a ridge term. We refer to such an
approximation method as the \emph{matrix ridge approximation} due to its direct motivation from the ridge regression model~\citep{HoerlRR:1970}.
The  approximation is built on an optimization problem with linear constraints.
This problem can be in turn solved by using the conventional spectral decomposition technique or an efficient iterative method.

Although the idea behind the matrix ridge approximation is simple, our method is attractive.
Firstly, it yields an approximation tighter than
the incomplete Cholesky decomposition and the incomplete spectral decomposition do.
Secondly, it yields an approximate matrix, whose condition number is not higher than that of the original matrix. This can
make numerical computations involved more stable.
More importantly, it can widen the application range of the low-rank approximation approach. Particularly,
we show that our method can be applied to the approximate computation of the inverse and spectral decomposition of a positive (semi)definite matrix.
We illustrate the application of the matrix ridge approximation
in spectral clustering and Gaussian process regression.

We also discuss two statistical
counterparts for the ridge
approximation. The first counterpart is in the spirit of  probabilistic interpretations of some
machine learning methods, including probabilistic PCA~\citep{Tipping:1999,RoweisNIPS:1998,Ahn:2003},
probabilistic nonlinear component analysis~\citep{RosipalNC:2001}, Gaussian process latent variable models~\citep{LawrenceNIPS:2003},
and factor
analysis~\citep{Magnus:1999}. In particular, we define a normal latent variable model in which
we impose the linear constraints.
Based on the latent variable model, we devise an iterative method, i.e., the expectation-maximization (EM) algorithm~\citep{Dempster:1977}, for
solving the matrix  ridge approximation problem.

The second counterpart is a Wishart model, which is derived from the normal latent variable
model by using the relationship between Wishart distributions and Gaussian distributions~\citep{GuptaN:Book:2000,ZhangJML:2006}.
These two statistical counterparts in turn define
probabilistic matrix ridge approximation models.
Moreover, we show
that the maximum likelihood  approach to estimating the
parameters of the probabilistic models results in the same solution as that based on
the standard spectral decomposition technique.

The remainder of the paper is organized as follows. We first give the notation in Section~\ref{sec:notation}.
We present the matrix ridge approximation in Section~\ref{sec:golub} and illustrate its applications
in Section~\ref{sec:app}.  We
reformulate the matrix ridge approximation by using a normal latent variable model and a Wishart model in Section~\ref{sec:pra}.
Consequently, we develop probabilistic ridge approximation and an
EM iterative algorithm.  Section~\ref{sec:exp}
conducts the empirical analysis, and Section~\ref{sec:conclusion}  concludes our work.
Note that we put all proofs to the appendices.

\section{Notation and Terminology}
\label{sec:notation}

We let $\I_m$ denote the $m{\times}m$ identity matrix, and $\1_m$ denote the  $m{\times}1$ vector of ones.
For a matrix $\Y$, we denote its rank, Frobenius norm and condition number by $\rk(\Y)$, $\|\Y\|_F$ and $\kappa(\Y)$, respectively. When $\Y$ is square, we
denote its  determinant and trace by $|\Y|$ and $\tr(\Y)$.
Additionally, $\A \otimes \B$ denotes the Kronecker
product of $\A$ and $\B$.

For an $s{\times}t$ random matrix $\Z$, $\Z \thicksim
N_{s, t}(\M, \A{\otimes}\B)$ means that $\Z=[z_{ij}]$ ($s{\times}t$) follows a
matrix-variate normal distribution with mean matrix $\M=[m_{ij}]$
($s{\times}t$) and covariance matrix $\A{\otimes}\B$, where $\A$
($s{\times}s$) and $\B$ ($t{\times}t$) are  symmetric positive definite.
Note that a matrix variate normal distribution is defined through a multivariate normal distribution~\citep{GuptaN:Book:2000}.
In particular, let $\vect(\Z^T)=(z_{11}, \ldots, z_{1t}, z_{21}, \ldots, z_{st})^T$ ($st{\times}1$)
and $\vect(\M^T)=(m_{11}, \ldots, m_{1t}, m_{21}, \ldots, m_{st})^T$ ($st{\times}1$). Then, $\Z \thicksim N_{s, t}(\M, \A{\otimes}\B)$
if and only if $\vect(\Z^T) \thicksim N_{st}(\vect(\M^T), \A{\otimes}\B)$.
We also use the notation in~\cite{GuptaN:Book:2000} for Wishart distributions. That is,
for an $m{\times}m$ positive definite random $\Y$, $\Y\thicksim W_m(r, \Si)$ represents that $\Y$ follows a Wishart distribution with degree of freedom $r$.

Finally, in Table~\ref{tab:nota} we list some notations
that will be used throughout this paper. It is clear that $\H_b\H_b=\H_b$ and $\BP\BP=\BP$; i.e., they
are idempotent. Moreover, we have $\H_b\1_m=\0$, $\b^T \H_b=\0$,
$\BP\b=\0$ and $\b^T\BP=\0$. A typical nonzero case for $\b$ is
$\b= \frac{1}{\sqrt{m}}\1_m$. This case implies that
$\H_b=\BP=\I_m{-}\frac{1}{m}\1_m \1_m^T$ and $\A^T\1_m=\0$ (that is,
the mean of the rows of $\A$ is zero). In addition, let us keep in
mind that $\H_b= \BP=\I_m$ when $\b=\0$ for notational simplicity.
In this case we always have $\M=\BS=\T$.

\begin{table}[!ht]
\caption{Some notations that will be used in this paper.} \label{tab:nota}
\begin{center}
\begin{tabular}{|l|l|}
\hline
 $\b \in \BR_{+}^{m}$ & a $m$-dimensional nonnegative vector \\
 $\M \in \SR_{+}^{m{\times}m}$ & a positive semidefinite matrix of rank $p$ ($1 <p \leq m$) \\
 $\H_b=\I_m {-} \frac{\1_m \b^T} {\1_m^T\b} $ ($\b\neq \0$) & centering matrix  \\
 $\BP=\I_m {-}  \b \b^T$ & projection matrix \\
 $\BS=\H_b \M \H_b^T$ & positive semidefinite  matrix \\  $\T=\BP \M
\BP$ & positive semidefinite  matrix \\  \hline
\end{tabular}
\end{center}
\end{table}

\section{The Matrix Ridge Approximation}
\label{sec:golub}

We are given a nonnegative vector $\b \in \BR_{+}^{m}$  and a positive semidefinite
matrix $\M \in \SR_{+}^{m{\times}m}$ of rank $p$ ($1 <p \leq m$). The ridge approximation of $\M$
is defined as
\[
\M \thickapprox \A \A^T + \delta \I_m,
\]
where $\delta > 0$ is called a \emph{ridge term}, and $\A \in \BR^{m{\times}q}$ is a matrix of
full column rank $q$ ($< p$) and satisfies $\A^T \b=\0$. The idea behind the matrix ridge approximation is simple,
and the terminology is motivated by the ridge regression model~\citep{HoerlRR:1970}.
Note that when $\b=\0$, $\A^T\b=\0$ is always true. This implies
no constraints. In this paper we consider both the cases with and without the linear constraints.
Since $\A^T \b=\0$ is equivalent to $c \A^T\b=\0$ for any nonzero constant $c$,
we assume that
$\b^T\b =1$ whenever $\b \neq \0$ to make the constraint identifiable.

The constraint $\A^T \b =\0$ for $\b \neq \0$ is often met
in machine learning methods such as the classical multidimensional scaling~\citep{Gower:1986},
kernel PCA~\citep{ScholkopfBook:2002}, spectral clustering~\citep{ZhangzhSTS:2008a}, etc.
If $\b=\0$ and $\delta=0$, we obtain the incomplete  factorization
$\M \thickapprox \A \A^T$ straightforwardly by using the spectral
decomposition of $\M$~\citep{Magnus:1999}. In this setting, the
ridge approximation is also closely related to the incomplete Cholesky
factorization~\citep{Golub:1996}.  Furthermore, if
$q=p$ it is feasible to obtain
an exact expression $\M = \A \A^T$ via the spectral (or Cholesky) decomposition.
In this paper we concentrate on the case that $q<p$ and
$\delta>0$, so we have a sparse plus low-rank approximation of $\M$ ($\delta \I_m$ is sparse and $\A \A^T$ is low-rank).

In order to estimate $\A$ and $\delta$, we exploit two loss
functions which were developed for estimation of covariance
matrices~\citep{Anderson:1984}. In particular, the first loss function is a
least-squares error:
\[
F(\A, \delta) = \tr\big[(\BS {-} \A \A^T {-} \delta \I_m)^2\big]
\]
while the second loss is derived from the likelihood function;
namely,
\[
G(\A, \delta) = \log |\A \A^T + \delta \I_m| + \tr[( \A \A^T +
\delta \I_m)^{-1} \BS].
\]

\begin{theorem} \label{thm:lse2}
Let $\gamma_1\geq \cdots \geq \gamma_q \geq \cdots \geq \gamma_{m}$
\emph{($\geq 0$)} be the eigenvalues of $\BS= \H_b \M \H_b^T$, $\V$
be an arbitrary $q{\times}q$ orthogonal  matrix, ${\Gam_q}$ be a
$q{\times}q$ diagonal matrix containing the first $q$ principal
(largest) eigenvalues $\gamma_i$, and ${{\U_q}}$ be an $n{\times}q$
column-orthonormal matrix in which the $q$ column vectors are the principal
eigenvectors corresponding to ${\Gam_q}$. Assume that $\delta>0$ and
that $\A \in \BR^{m{\times}q}$ \emph{($q< \min(m, p)$)} is of full
column rank and satisfies $\A^T\b=\0$. If there exists a $j
\in\{q{+}1, \ldots, m\}$ such that  $\gamma_q> \gamma_j>0$, then the strict local
minimum of $F(\A, \delta)$ and of $G(\A, \delta)$ with respect to (w.r.t.) $(\A,
\delta)$ is obtained when
\[
\widehat{\A} = {\U_q} ({\Gam_q} - \hat{\delta} \I_q)^{1/2} \V \quad
\mbox{and} \quad  \hat{\delta} = \frac{1}{m{-}q}
\sum_{j=q+1}^{m}\gamma_j.
\]
\end{theorem}

Theorem~\ref{thm:lse2} is a direct corollary of Theorem~\ref{thm:lse}
in Appendix~\ref{ap:a}. Theorem~\ref{thm:lse2} also shows that the
minimizer $(\widehat{\A}, \hat{\delta})$ of $F(\A, \delta)$ is the
same to that of $G(\A, \delta)$. We consider the case that $\b=0$. In this case,
the condition number of $\M$ ($=\BS$) is
$\kappa(\M)=\frac{\gamma_1}{\gamma_m}$. It follows from Theorem~\ref{thm:lse2}
that $\kappa(\widehat{\A} \widehat{\A}^T {+} \hat{\delta} \I_m) = \frac{ \gamma_1}{\frac{1}{m{-}q} \sum_{j=q+1}^{m}\gamma_j } \leq \kappa(\M)$.
This implies that $\widehat{\A} \widehat{\A}^T {+} \hat{\delta}\I_m$ is well-conditioned more than $\M$~\citep{Golub:1996}.
In other words, if $\M$ is well-conditioned, so is $\widehat{\A} \widehat{\A}^T {+} \hat{\delta}\I_m$.

In addition, it is easily
calculated that
\[
F(\widehat{\A}, \hat{\delta}) = \sum_{i=q{+}1}^m \gamma_i^2 -
\frac{1}{m{-}q} \Big(\sum_{i=q{+}1}^m \gamma_i \Big)^2.
\]
It is well known that
\[
\inf_{\begin{array}{c} \B \in \BR^{m{\times}m} \\ \rk(\B)\leq q \end{array}} \|\BS-\B\|_{F}^2
= \inf_{\begin{array}{c} \Y \in \BR^{m{\times}q} \\ \rk(\Y)\leq q \end{array}} \|\BS-\Y \Y^T\|_{F}^2 = \sum_{i=q{+}1}^m \gamma_i^2.
\]
Thus, when comparing the ridge approximation of $\BS$ with the
incomplete Cholesky decomposition
of $\BS$, we have
\[
\inf_{\begin{array}{c} \L \in {\cal L} \\ \rk(\L)= q \end{array}} \|\BS{-}\L \L^T\|_{F}^2 \geq
\inf_{\begin{array}{c} \Y \in \BR^{m{\times}q} \\ \rk(\Y)= q \end{array}} \|\BS{-}\Y \Y^T\|_{F}^2 \geq
\inf_{\begin{array}{c} \delta\geq 0, \A \in \BR^{m{\times}q} \\ \rk(\A)= q \end{array}} \|\BS{-}\A \A^T{-} \delta \I_m\|_{F}^2,
\]
where ${\cal L}=\{\L \in \BR^{m{\times}q}: \L \mbox{ is lower triangular} \}$.
This shows that the ridge approximation yields a tighter approximation of $\BS$ than both
the incomplete Cholesky decomposition and the incomplete spectral decomposition do.

As we mentioned,  $G(\A, \delta)$ is
associated with a likelihood function.  In Section~\ref{sec:nlvm} we
will show that  $G$ is derived from a normal latent variable
model. Thus, the solution in Theorem~\ref{thm:lse2} is in fact the conventional maximum likelihood (ML) estimate.
Furthermore, the ML estimation method is based on the direct spectral decomposition
of the $m{\times}m$ matrix $\M$ or $\BS$, which takes $O(m^3)$ operations.
Thus, the method is inefficient when $m$ is very  large.
Based on the idea behind the latent variable model, we develop an
iterative method for solving the matrix ridge approximation.

In particular, given
the $t$th estimates $\A_{(t)}$ and $\delta_{(t)}$ of $\A$ and $\delta$,
the next estimates  of $\A$ and $\delta$ in our iterative method are given as:
\begin{eqnarray}
\A_{(t{+}1)} &=& \BS \A_{(t)}  \big( \delta_{(t)}\I_q + \Si_{(t)}^{-1}
\A_{(t)}^T  \BS \A_{(t)} \big)^{-1},
\label{eq:em1} \\
\delta_{(t{+}1)} &=& \frac{1}{m} \Big[ \tr(\BS) - \tr\big(\A_{(t{+}1)}
\Si_{(t)}^{-1} \A_{(t)}^T \BS \big) \Big],  \label{eq:em2}
\end{eqnarray}
where $\Si_{(t)} = \delta_{(t)}\I_q + \A_{(t)}^T \A_{(t)}$. Derivation of the
algorithm is given in Section~\ref{sec:nlvm} and Appendix~\ref{ap:ff}. This procedure involves  multiplication of  $m{\times}m$ matrices by  $m{\times}q$ matrices and inversion of $q{\times}q$ matrices.  Inverting a $q{\times}q$ matrix takes $O(q^3)$ operations, and  multiplying an $m{\times}m$ matrix by an $m{\times}q$ matrix runs in
$m^2 q$ flops.  Thus, this method takes time $O(T m^2 q)$, where $T$ is the maximum iterative number.
The method is efficient because
$T$ is usually far smaller than $m$ (even smaller than $\sqrt{m}$), especially when $m$ is vary large. In the following experiment,
we will see that  in most cases the EM iterations  get convergence after about 20 steps.
Moreover, the matrix multiplication can be easily implemented in parallel.
Additionally, the EM method does not necessarily load whole $m{\times}m$ matrix $\BS$ during the iterations,
which can significantly reduce the storage space.

Given an initial matrix $\A_{(0)}$ such that $\ran(\A_{(0)}) \subseteq \ran(\BS)$ where
$\ran(\Z)$ represents the space spanned by the columns of $\Z$, we have the following lemma.

\begin{lemma} \label{lem:1} Assume that the matrices
$\{\A_{(t)}\}$ are generated by (\ref{eq:em1}) and (\ref{eq:em2}). If $\ran(\A_{(0)}) \subseteq \ran(\BS)$ and $\rk(\A_{(0)})=q$, then
for all $t>1$, the $\A_{(t)}$ are of full column rank.
\end{lemma}

In Section~\ref{sec:nlvm} we will show that the iterative method given in (\ref{eq:em1}) and (\ref{eq:em2}) is
a standard EM iterative procedure~\citep{Dempster:1977}. Consequently,
its convergence has been well established~\citep{WuEM:1983}.
The following theorem proves that the constraints $\A_{(t)}^T\b=\0$ always hold during the iteration procedure and
the EM estimates converge
to the corresponding ML estimates. In other words, the EM iteration converges to the strict local minimizer.

\begin{theorem} \label{thm:2}
Given initial values $\delta_{(0)}$ and $\A_{(0)}$ subject to $\delta_{(0)}>0$
and $\A_{(0)}^T\b=\0$, the values of $\A_{(t)}$ and $\delta_{(t)}$ calculated via
(\ref{eq:em1}) and (\ref{eq:em2}) always satisfy $\A_{(t)}^T\b =\0$ and
$\delta_{(t)}>0$. Moreover, the EM estimates of $\A$ and $\delta$
converge to the conventional ML solutions given in Theorem~\ref{thm:lse2}.
\end{theorem}

The EM
algorithm provides an efficient iterative method for computing the matrix ridge
approximation. This iterative method is related to the power method
and the Lanczos method~\citep{Golub:1996}, which typically serve for solving matrix eigenvector
problems numerically.
Specifically, this EM algorithm is similar to the QR orthogonal iteration, which is a straightforward generalization of the power method to find a higher-dimensional invariant
subspace~\citep{Golub:1996}.

Intuitively, it seems interesting that we consider a  two-step procedure to solve the matrix ridge approximation as follows. Specifically,
we first apply the  QR orthogonal iteration to
obtain an $m{\times}q$ column-orthonormal matrix $\Q$ and set $\widehat{\A}=\Q(\Q^T\BS\Q)^{1/2}$.  We then
calculate $\hat{\delta}=\frac{\tr(\BS)- \tr(\Q^T\BS\Q)}{m}$ based on the minimization of $\tr((\BS- \widehat{\A}\widehat{\A}^T-\delta \I_m)^2)$ w.r.t.\ $\delta$.
Assume that $\Q^T\BS\Q=\Gam_q$. Then $\hat{\delta}=\frac{\sum_{i=q{+}1}^m \gamma_i}{m}$.
It is directly computed that
\[
\tr((\BS- \widehat{\A}\widehat{\A}^T- \hat{\delta} \I_m)^2) = \sum_{i=q{+}1}\gamma_i^2
- \frac{1}{m}\Big(\sum_{i=q{+}1}\gamma_i\Big)^2 > \sum_{i=q{+}1}\gamma_i^2
- \frac{1}{m-q}\Big(\sum_{i=q{+}1}\gamma_i\Big)^2.
\]
This implies that the two-step procedure can not find the optimum solution of the matrix ridge approximation
problem. Moreover, we have $\kappa(\widehat{\A} \widehat{\A}^T + \hat{\delta} \I_m) = 1+ \frac{m \gamma_1}{\sum_{i=q{+}1}^m \gamma_i}$.
Compared with our method, this two-step method results in the approximation with higher condition number. Moreover, the method can not keep the well-conditionedness  of the original matrix (if it is well-conditioned).
We will conduct simulation on a toy data in Section~\ref{sec:exp1}, which shows that the two-step  method fails to solve the matrix ridge approximation problem.

It is worth noting that the nonnegativity on $\b$ is not
necessary in our derivation for the estimation methods. In fact, we
are able to extend the constraints $\A^T\b=\0$ to $\A^T\E=\0$ where
$\E$ is an $m{\times}k$ matrix of full column rank ($k{+}p\leq m$). In
this case, letting $\BP=\I_m{-} \E (\E^T\E)^{-1} \E^T$ and $\BS=\BP \M
\BP$, we alternatively use $\tr(\BS- \A \A^T - \delta \I_m)^2$ as the
loss function. The resulting solution is also similar to that in
Theorem~\ref{thm:lse2}.

\section{Applications of the Matrix Ridge Approximation}
\label{sec:app}

The matrix ridge approximation has potential applications in multivariate analysis and
machine learning. In this section we present two
important examples to illustrate its applications.

Let $\M$ be an $m{\times}m$ symmetric positive (semi)definite
matrix. It is well known that the computational complexities of
calculating the inverse of $\M$ and the spectral decomposition
of $\M$ are $O(m^3)$. Thus, the computational costs are high when
$m$ is large. We now address these two computational issues via the matrix ridge approximation. First of all,
assume we obtain that $\M \approx \delta \I_m+ \A
\A^T$ where $\A \in \BR^{m{\times}q}$ and $q\ll m$  using the EM
iteration.

In the first example we consider the computation of
$\M^{-1}$ where $\M$ is positive definite. We approximate $\M^{-1}$ by $(\delta \I_m+ \A
\A^T)^{-1}$ which is then calculated by using the
Sherman-Morrison-Woodbury formula; i.e.,
\begin{equation} \label{eq:smw}
(\delta \I_m+ \A \A^T)^{-1}= \delta^{-1} \I_m- \delta^{-1} \A(\delta
\I_q + \A^T\A)^{-1} \A^T.
\end{equation}
Clearly, the current complexity is $O(m q^2)$.
Thus, the computational cost will become much lower when $q$ is far less than $m$.

Recall that the incomplete Cholesky decomposition is widely used in
the literature. For the $m{\times}m$ positive definite matrix
$\M$, we can consider its approximation by using the incomplete
Cholesky decomposition, that is, $\M\thickapprox \L \L^T$ where $\L$ is an
$m{\times}q$ lower triangular matrix. 
Since $\L \L^T$ is singular, this decomposition
can not directly provide  us an approach to the
approximation of $\M^{-1}$. Also, the Nystr\"{o}m
method could not be directly used for  the
approximation of $\M^{-1}$. We can employ the two-step procedure as discussed in the previous
section. However, we have also shown that this two-step procedure can not find the optimum solution,
which will be empirically illustrated in Section~\ref{sec:exp1}.

If $\M$ has an explicit form of
\begin{equation} \label{eqn:spe_str}
\M = \Ph + \alpha^2 \I_m
\end{equation}
where $\Ph$ is an available $m{\times}m$ positive semidefinite matrix and $\alpha\neq 0$ is prespecified, both the
incomplete Cholesky decomposition and the Nystr\"{o}m
method work. Specifically,  one first implements either the
incomplete Cholesky decomposition or the Nystr\"{o}m
method on $\Ph$ to obtain $\L$ and then uses the Sherman-Morrison-Woodbury formula.
Since our method directly applies to $\M$ (rather than $\Ph$), our method can obtain a tighter approximation to $\M$.
Consider that the ridge term in our method $\delta$ is $\alpha^2 + \frac{1}{m{-}q} \sum_{j=q{+}1}^m \lambda_j$
where  $\lambda_{1}\geq \lambda_2 \geq \cdots \geq \lambda_m$  are the  eigenvalues of $\Ph$.
The condition number of the approximate matrix with our method is
\[ \kappa(\A\A^T + \delta \I_m) = \frac{\alpha^2+ \lambda_1 }{\alpha^2 + \frac{1}{m{-}q} \sum_{j=q{+}1}^m \lambda_j},
\]
while the condition number  with the incomplete Cholesky decomposition is
\[\kappa(\L \L^T + \alpha^2 \I_m) = \frac{\alpha^2+ \lambda_1 }{\alpha^2}.
\]
Therefore, our method is more  stable numerically especially when $\alpha^2$ takes a very small value.
Our simulation in Section~\ref{sec:exp1}  further illustrates the issues.
We will see that when $\alpha^2$ takes a very small value, the incomplete Cholesky decomposition
fails to approximate the inversion of $\M$.

We note that any strictly positive definite matrix $\M$ can be expressed as in (\ref{eqn:spe_str}).
For example, we take $\alpha^2$ as the smallest eigenvalue of $\M$. In this case,
it is required to estimate the smallest eigenvalue prior to the implementation of the incomplete Cholesky decomposition (or the Nystr\"{o}m
method). Thus, the method becomes inefficient. Moreover, the previous issues still exist for the incomplete Cholesky decomposition
and the Nystr\"{o}m method in comparison with our method.

In the second example, we are concerned with the symmetric eigenvector problem, which plays an important role
in multivariate statistical analysis and machine learning.  The  eigenvector problem is defined by
\begin{align} \label{eq:sep}
& \max_{\X \in \BR^{m{\times}q}} \; \frac{1}{2} \tr(\X^T \M \X) \\
& \mbox{subject to } \quad \X^T \X = \I_q \mbox{ and } \X^T \b=\0.
\nonumber
\end{align}
If $\b=\0$, Problem~(\ref{eq:sep}) becomes the standard
Rayleigh quotient problem~\citep{Golub:1996}.
Furthermore, if viewing $\M$ as a sample covariance matrix, it is equivalent to
the PCA problem~\citep{Jolliffe:2002}.

If $\b^T \b =1$, the problem in~(\ref{eq:sep}) is
a symmetric eigenvector problem with linear
constraints~\citep{GolubSIAM:1973}.
It defines a spectral clustering problem when $\M$ is set as a kernel matrix~\citep{ZhangzhSTS:2008a}.

Consider the spectral decomposition (or singular value
decomposition, SVD) of $\T = \BP \M \BP$ as $\T = \U^T \Gam \U$ where
$\U \in \BR^{m{\times}m}$ is orthogonal and $\Gam=\diag(\gamma_1,
\ldots, \gamma_m)$ is arranged in descending order.
Let $\widehat{\X}=\U_{q} \V$, where $\U_{q}$ is the
$m{\times}q$ matrix containing the first $q$ columns of $\U$ and
$\V$ is an arbitrary $q{\times}q$ orthogonal matrix. Then the matrix
$\widehat{\X}$ is  the maximizer of the eigenvector problem in
(\ref{eq:sep})~\citep[see,][]{GolubSIAM:1973}. On the other hand,
it follows from
Theorem~\ref{thm:lse2} that $\hat{\A} (\hat{\A}^T
\hat{\A})^{-\frac{1}{2}}= \U_q \V = \widehat{\X}$. This implies that we can obtain the solution of (\ref{eq:sep})  via the matrix ridge approximation.

Note that if $\b=\0$,
$\hat{\A} (\hat{\A}^T \hat{\A})^{-\frac{1}{2}}$ spans the same subspace as that
spanned by the first $q$ principal eigenvectors of $\M$ ($=\T$). When $q=1$, $\hat{\A} (\hat{\A}^T \hat{\A})^{-\frac{1}{2}}$ is the
top eigenvector of $\M$. In this case, the EM iteration bears
resemblance to the power method~\citep{Golub:1996}.

Naturally and intuitively, the incomplete Cholesky decomposition method might be used to approximate the solution of the problem in~(\ref{eq:sep}).
Specifically, one first finds the incomplete Cholesky decomposition of $\T$ as $\T \approx \L \L^T$
and then treats $\L (\L^T\L)^{-\frac{1}{2}}$ as the solution of Problem~(\ref{eq:sep}).
However, to our knowledge,  in the existing literature there is no  theoretical guarantee  that  $\L (\L^T\L)^{-\frac{1}{2}}$ is
a solution of Problem~(\ref{eq:sep}). In fact, our experimental results in Section~\ref{sec:exp1}
show that  the incomplete Cholesky decomposition method is not appropriate to approximate
the solution of the problem in~(\ref{eq:sep}).


\section{Probabilistic Matrix Ridge Approximation Models}
\label{sec:pra}

In this section we consider two
probabilistic models for the matrix ridge approximation. We thus show
that the ML estimation approach for the
parameters of the probabilistic models results in the same solution as that based on the
standard spectral decomposition technique. The probabilistic formulation
also gives rise to the EM iterative method defined in~(\ref{eq:em1}) and (\ref{eq:em2}).

\subsection{The Normal Latent Variable Model}
\label{sec:nlvm}

In order to derive the EM iteration, we consider a probabilistic formulation of the matrix ridge approximation.
Our work is directly motivated by existing probabilistic interpretations of
dimensionality reduction  methods, such as probabilistic PCA~\citep{Tipping:1999,RoweisNIPS:1998,Ahn:2003},
probabilistic nonlinear component analysis~\citep{RosipalNC:2001}, Gaussian process latent variable
models~\citep{LawrenceNIPS:2003} and factor
analysis~\citep{Magnus:1999}.

Since $\M$ is an $m{\times}m$ positive semidefinite matrix of rank
$p$, there always exists an $m{\times}r$ matrix $\F$ with $r\geq p$
such that $\M=\F \F^T$.  Thus, we model $\F$  as a normal latent
variable model in matrix form:
\begin{equation} \label{eq:matr_version1}
\F = \A \W + \1_m \u^T + \Ups,
\end{equation}
where $\u$ is an $r{\times}1$ mean vector, $\W$ is a $q{\times}r$
latent matrix, and $\Ups$ is an $m{\times} r$ error matrix.
Furthermore, we assume
\begin{equation} \label{eq:matr_priors1}
\W \thicksim N_{q,r}\left(\0, \; (\I_q {\otimes} \I_r)/r \right)
\quad \mbox{ and } \quad \Ups \thicksim N_{m, r}\left(\0, \; (\delta
\I_m{\otimes}\I_r)/r \right),
\end{equation}
where $\delta >0$.

Typically, only $\M$ is available while both $r$ and $\F$ are unknown
in our case. Fortunately, we will see that our model can work via
some matrix tricks to yield an estimation procedure for the unknown
parameters $\A$ and $\delta$, which does not explicitly depend on
$r$ and $\F$.

It is clear that $\F \thicksim N_{m, r}(\1_m \u^T, \; (\A \A^T
{+}\delta \I_m) {\otimes} \I_r/r)$. Thus, the log-likelihood is
\begin{align*}
L(\A, \delta, \u) &= -\frac{m r}{2} \log(2 \pi)+\frac{m r}{2} \log r
-\frac{r}{2}\log|\Oma| -\frac{r}{2} \tr((\F{-}\1_m
\u^T)^T\Oma^{-1}(\F{-}\1_m \u^T)) \\
& \varpropto - \log|\Oma| - \tr((\F{-}\1_m \u^T)^T\Oma^{-1}(\F{-}\1_m
\u^T))
\end{align*}
where $\Oma=\A\A^T{+} \delta\I_m$.

We consider two setups for the mean vector $\u$. In the first
setup we let $\u=\0$. We then see that maximizing $L(\A, \delta,
\0)$ is equivalent to minimizing $G_1(\A, \delta)=\log|\Oma| +
\tr(\Oma^{-1} \M)$ where $\M=\F\F^T$, w.r.t.\ $(\A, \delta)$ under the constraint
$\A^T\b=\0$.
In the second setup we let $\u=\frac{1}{\1_m^T \b} \F^T \b$.
Substituting such a $\u$ in $L(\A, \delta, \u)$ leads to the conclusion
that the maximum likelihood estimate is  equivalent to minimizing
$G(\A, \delta)=\log|\Oma| + \tr(\Oma^{-1} \BS)$ w.r.t.\ $(\A,
\delta)$ under the constraint $\A^T\b=\0$. Thus, the matrix ridge approximation can also be
solved from  the probabilistic formulation.

Since our probabilistic model defined by
(\ref{eq:matr_version1}) and (\ref{eq:matr_priors1}) is a latent
variable model, this encourages us to develop an EM algorithm for
the parameter estimation. In particular, considering $\W$ as the
missing data, $\{\W, \F\}$ as the complete data, and $\A$ and
$\delta$ as the model parameters, we now have the EM algorithm for
the matrix ridge approximation, which is given in (\ref{eq:em1})
and (\ref{eq:em2}).  The derivation is then given in Appendix~\ref{ap:ff}. The algorithm is related to the EM algorithm
derived in the literature~\citep{RoweisNIPS:1998,Tipping:1999}.
However, we impose the constraint $\A^T\b=\0$ in our model.

\subsection{The Wishart Model}
\label{sec:wishart}

In this subsection we further explore the statistical properties of the matrix
ridge approximation. In particular, we establish a Wishart model,
corresponding to the treatments in the maximum likelihood estimation method.

First, we assume $\u=0$. We then have $\F \thicksim N_{m,r}\big(\0,
\; (\A \A^T {+} \delta \I_m) {\otimes} \I_r/r \big)$. Consequently, $\M=\F \F^T$ follows
Wishart distribution $W_m\left(r, (\A \A^T {+} \delta \I_m)/r
\right)$.
Second, it follows from (\ref{eq:matr_version1}) that $\F{-}\1_m
\u^T|\W \thicksim N_{m, r}\big(\A\W, \; \delta (\I_m{\otimes} \I_r)/r
\big)$. Hence,
\[
\F{-}\1_m \u^T \thicksim N_{m, r}\big(\0, \; (\A \A^T {+} \delta \I_m)
{\otimes} \I_r/r \big).
\]
Subsequently, $(\F{-}\1_m \u^T)(\F{-} \1_m \u^T)^T$ is distributed
according to  $W_m(r, \; (\A \A^T {+} \delta \I_m)/r)$. When $\u=
\frac{1}{\1_m^T\b} \F^T\b$, we thus have $(\F{-}\1_m \u^T)(\F{-} \1_m
\u^T)^T=\BS \thicksim W_m(r, \; (\A \A^T {+} \delta \I_m)/r)$.

Conversely, let $\M$ or $\BS$ follow a Wishart distribution
with an integral degree of freedom $r$. According to the equivalence  between
Gaussian  and Wishart distributions~\citep{GuptaN:Book:2000,ZhangJML:2006},  we can also obtain an
$m{\times}r$ matrix $\F$ which follows a matrix-variate normal
distribution.

In the normal latent variable and Wishart models, we assume that
$r$, the dimensionality of $\F$, is finite. In the reproducing
kernel literature~\citep{ScholkopfBook:2002}, $r$ is the
dimensionality of the feature space that can be infinite. For
example, the dimensionality of the feature space induced by the
Gaussian RBF kernel is infinite. In this case,
we study  the
\emph{asymptotic distribution} of $\BS$. Specifically, the asymptotic
distribution of $\frac{1}{\sqrt{r}} (\BS - (\A\A^T+\delta \I_m))$, as
$r \rightarrow \infty$, is a symmetric matrix-variate normal
distribution~\citep{GuptaN:Book:2000}.

It is worth pointing out that
the latent variable models provide a probabilistic formulation for PCO.
That is, it defines a probabilistic PCO model, which is dual to
the probabilistic PCA model~\citep{Tipping:1999}.

In the existing statistical approaches to multidimensional scaling
\citep{Ramsay:1982,Groenen:1995,Oh:2001}, an error structure of
$\delta_{ij}^2$ is incorporated so that $\delta_{ij}^2$, conditioned
on $d_{ij}^2=\|{\y}_i-{\y}_j\|^2$, has p.d.f.
$p(\delta_{ij}^2|d_{ij}^2)$. Since $\delta_{ij}$ must be
nonnegative, $\delta_{ij}^2$ is usually modeled as a
truncated normal or log-normal distribution, with parameters
$d_{ij}^2$. Moreover, the $\delta_{ij}^2$ are assumed to be
independent.
This provides an approach to the ML estimates of the ${\y}_i$.
Some numerical
methods such as gradient methods and Bayesian sampling methods such as MCMC
are then used to calculate the ${\y}_i$.

However, these statistical approaches are not  appropriate for
probabilistic modeling of PCO. Since the dissimilarity matrix
$\De=[\delta_{ij}^2]$ in PCO is Euclidean, the triangle inequality
\[
\delta_{ij} + \delta_{ik} \geq \delta_{jk}
\]
should be satisfied. On one hand, this makes a conflict with the assumption that
the $\delta_{ij}^2$ are independent. On the
other hand, for the $\delta_{ij}^2$ generated from a truncated
normal or log-normal distribution, the triangle inequality is no
longer guaranteed. Accordingly,  $\De$ is not
Euclidean. In our Wishart model the interactions among the
$\delta_{ij}^2$ are explored, because we treat the  similarity
matrix $\Q=-\frac{1}{2}\H_b \De \H_b^T$ as a Wishart matrix. Moreover, the positive
semidefiniteness of $\Q$ implies the Euclideanarity of $\De$~\citep[e.g. see,][]{Gower:1986}.

\section{Experiments}
\label{sec:exp}

As we see from Theorem~\ref{thm:lse2}, the conventional ML estimation approach
gives the same solution
as the corresponding least squares counterpart. Moreover, the  ML estimate
is obtained by using the standard direct spectral decomposition (SD)
technique. Our analysis has also provided an EM iterative algorithm. Thus, it
is of interest to compare the performance of the EM algorithm with
the direct SD method. All algorithms have been implemented in Matlab.

\subsection{Performance Analysis on Toy Datasets}
\label{sec:exp1}

In Section~\ref{sec:golub} we show that the EM algorithm is more
efficient than the SD method when $m$ is large. Moreover, the
solution of the EM algorithm converges to that of the conventional ML estimate
based on the SD method.  We performed our experimental analysis
based on a toy dataset by studying the two applications of the matrix
ridge approximation presented in Section~\ref{sec:app}.

In the simulation we used a $10{\times}10$ positive definite
matrix $\M$, which is given by
\setlength{\tabcolsep}{2.0pt}
\begin{small}
\[
 \M=
     \begin{bmatrix}    1.8147  &  0.8650 &   0.8781 &   0.8106  &  0.9900  &  0.8270 &   0.8737  &
         0.9851 & 0.6538  &  0.8958 \\
        0.8650   & 1.9058   & 0.9560  &  0.9465 &   0.8311  &  0.5516  &  0.8781
      &  0.9139  & 0.8781  &  0.9851 \\
        0.8781  &  0.9560  &  1.1270  &  0.9704  &  0.8781  &  0.5543  &  0.9656
      &  0.9185 & 0.9094   & 0.9512 \\
    0.8106  &  0.9465  &  0.9704  &  1.9134  &  0.8106  &  0.5066  &  0.9465
    & 0.8825 &  0.9560  &  0.9512 \\
     0.9900  &  0.8311  &  0.8781  &  0.8106  &  1.6324  &  0.8270  &  0.9003
    & 0.9753 &  0.6538  &  0.8694 \\
    0.8270   & 0.5516  &  0.5543  &  0.5066  &  0.8270  &  1.0975  &  0.6096
   & 0.7827  &  0.3447  &  0.6005 \\
    0.8737  &  0.8781  &  0.9656  &  0.9465  &  0.9003 &   0.6096  &  1.2785
   & 0.9139 &    0.8607  &  0.8914 \\
    0.9851  &  0.9139  &  0.9185  &  0.8825  &  0.9753  &  0.7827  &  0.9139
   & 1.5469 &   0.7334  &  0.9465 \\
    0.6538  &  0.8781 &   0.9094  &  0.9560  &  0.6538  &  0.3447  &  0.8607
   & 0.7334 &  1.9575 &   0.8564 \\
    0.8958  &  0.9851  &  0.9512  &  0.9512  &  0.8694  &  0.6005   & 0.8914
   & 0.9465 &  0.8564  &  1.9649 \end{bmatrix}.  
\]
\end{small}
The eigenvalues of $\M$ are  $9.2521$, $1.6413$, $1.0326$, $0.9460$,
$0.9386$, $0.7530$, $0.5925$, $0.4736$, $0.4142$  and  $0.1946$.
Thus, the eigenvalues of $\M^{-1}$ are 5.1387,  2.4143, 2.1115,
1.6878, 1.3280, 1.0654, 1.0571, 0.9684, 0.6093,  0.1081.  Our
current purpose is to approximate $\M^{-1}$ by using the ridge
approximation of $\M$.  That is, we first implemented the ML
estimates of $\A$ and $\delta$ and then calculated $(\A \A^T+ \delta
\I_m)^{-1}$---an approximation of $\M^{-1}$---in terms of
(\ref{eq:smw}). In this example, $\b=0$ which implies that there is no constraint for $\M$.

In the EM iteration we randomly generated 10 $q$-dimensional vectors
from $N_q(\0, \I_q)$ as the initial value $\A_{(0)}$ of $\A$ and set the
initial value of $\delta$ as $\delta_{(0)}=0.0001$. We implemented our
analysis for $q=1, \ldots, 9$. After taking about 20 step, the EM
iterations converge to the conventional ML solution based on the
spectral decomposition method. Table~\ref{tab:delta} reports the
SD-based ML estimates and the EM iteration estimates of $\delta$ for
$q=1, \ldots, 9$. The corresponding values are almost identical.

We evaluate the performance of $(\A\A^T+\delta \I_m)^{-1}$, as an
approximation of $\M^{-1}$,  by employing the following two criteria:
\[
e_F= \frac{1}{\sqrt{m}}\|\I_m - \M (\A\A^T+ \delta \I_m)^{-1}\|_{F} \quad \mbox{ and
} \quad e_2= \|\I_m - \M (\A \A^T + \delta \I_m)^{-1} \|_{2}.
\]
The $e_F$ and $e_2$ are given in Figure~\ref{fig:exam1}. We see
that $e_F$ and $e_2$ become small as $q$ increases. Especially, when
$q=9$, their values are $0.0024$ and $0.0076$   respectively. Moreover, in this case, the
eigenvalues of $(\A \A^T + \delta \I_m)^{-1}$ are 5.1381, 2.4143,
2.1115, 1.6877,  1.3281, 1.0654,  1.0571, 0.9684,  0.6093 and
0.1081, which are almost equal to those of $\M^{-1}$.

For comparison, we also performed the two-step method based on the QR orthogonal iteration (see Section~\ref{sec:golub}). We define the initial column-orthonormal
matrix $\Q_{(0)}=\A_{(0)} (\A_{(0)}^T\A_{(0)})^{-\frac{1}{2}}$ where  $\A_{(0)}$ is the same to that for the EM iteration.
As we see from Figure~\ref{fig:exam1}, for $q=1$ and $q=2$,
the two-step method has approximation errors similar to the EM iteration method. However, the two-step method fails to obtain a
good approximation in other cases. When $q=9$, the errors of the method are  $e_F=2.8462$ and $e_2=9.00$.
Additionally, the
eigenvalues of $(\A \A^T + \delta \I_m)^{-1}$ with the QR iteration are 51.2821, 2.3057, 2.0281,  1.6339, 1.2945,
1.0437, 1.0357,  0.9505, 0.6021, and 0.1079.

\begin{table}[!ht]
\begin{center}
\caption{The estimated values of $\delta$ with the ML based on the SD and
EM iteration methods.} \label{tab:delta}
\begin{tabular}{|l|ccccccccc|} \hline
          &  $q=1$ & $q=2$  & $q=3$ & $q=4$ & $q=5$ & $q=6$ & $q=7$ & $q=8$ & $q=9$ \\ \hline
{\tt SD} & 0.7763  &  0.6681 & 0.6161 & 0.5611 & 0.4856 & 0.4187 & 0.3608 & 0.3044 & 0.1946   \\
\hline
{\tt EM} &  0.7763& 0.6681 & 0.6161 & 0.5614 &  0.4856 & 0.4187 & 0.3608 & 0.3044 & 0.1945  \\
\hline
\end{tabular}
\end{center}
\end{table}

\begin{figure}[!ht]
\centering
\begin{tabular}{cc}
\includegraphics*[width=80mm,
height=60mm]{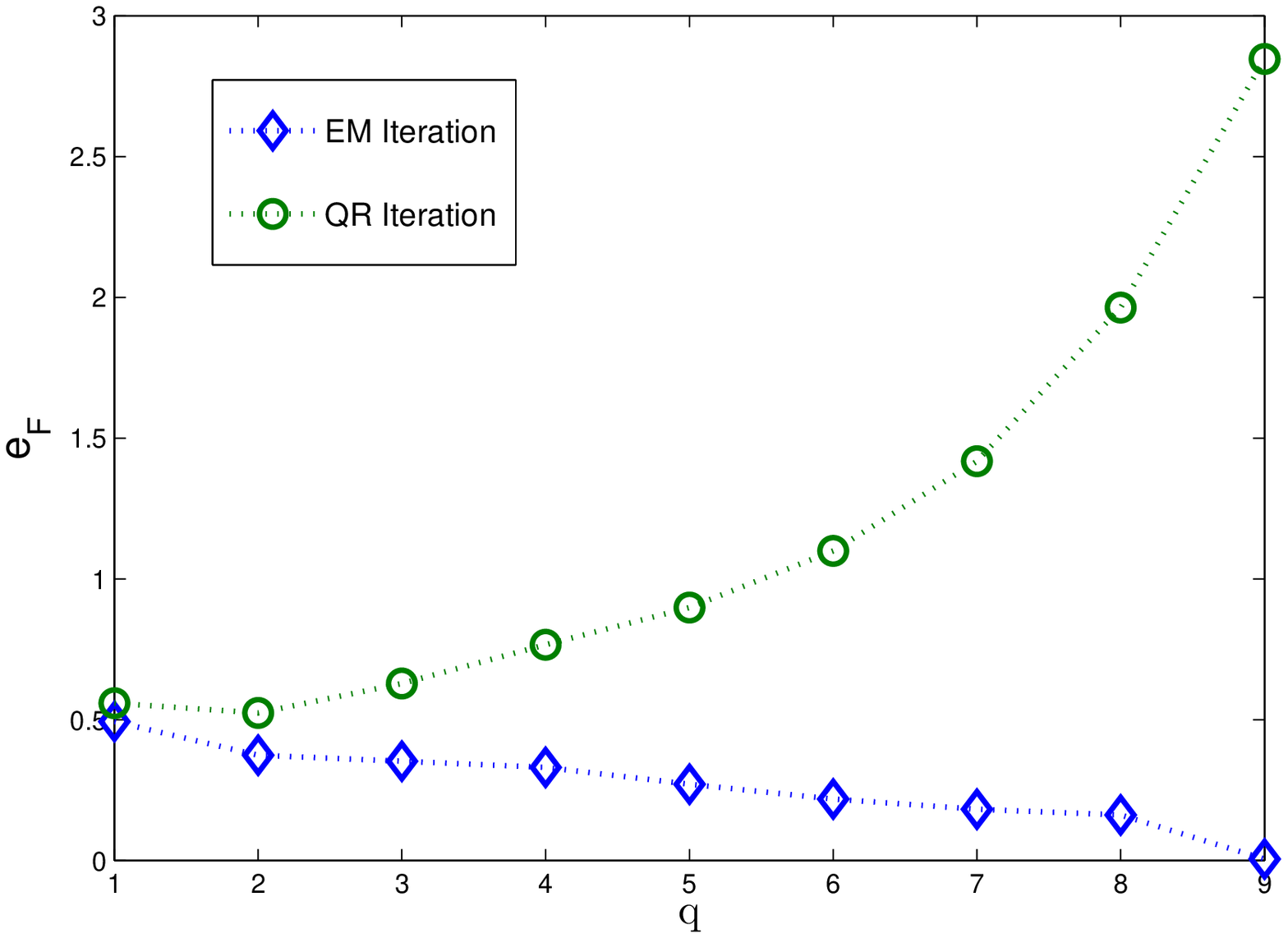}
& \hspace{-0.1in}  \includegraphics*[width=80mm, height=60mm]{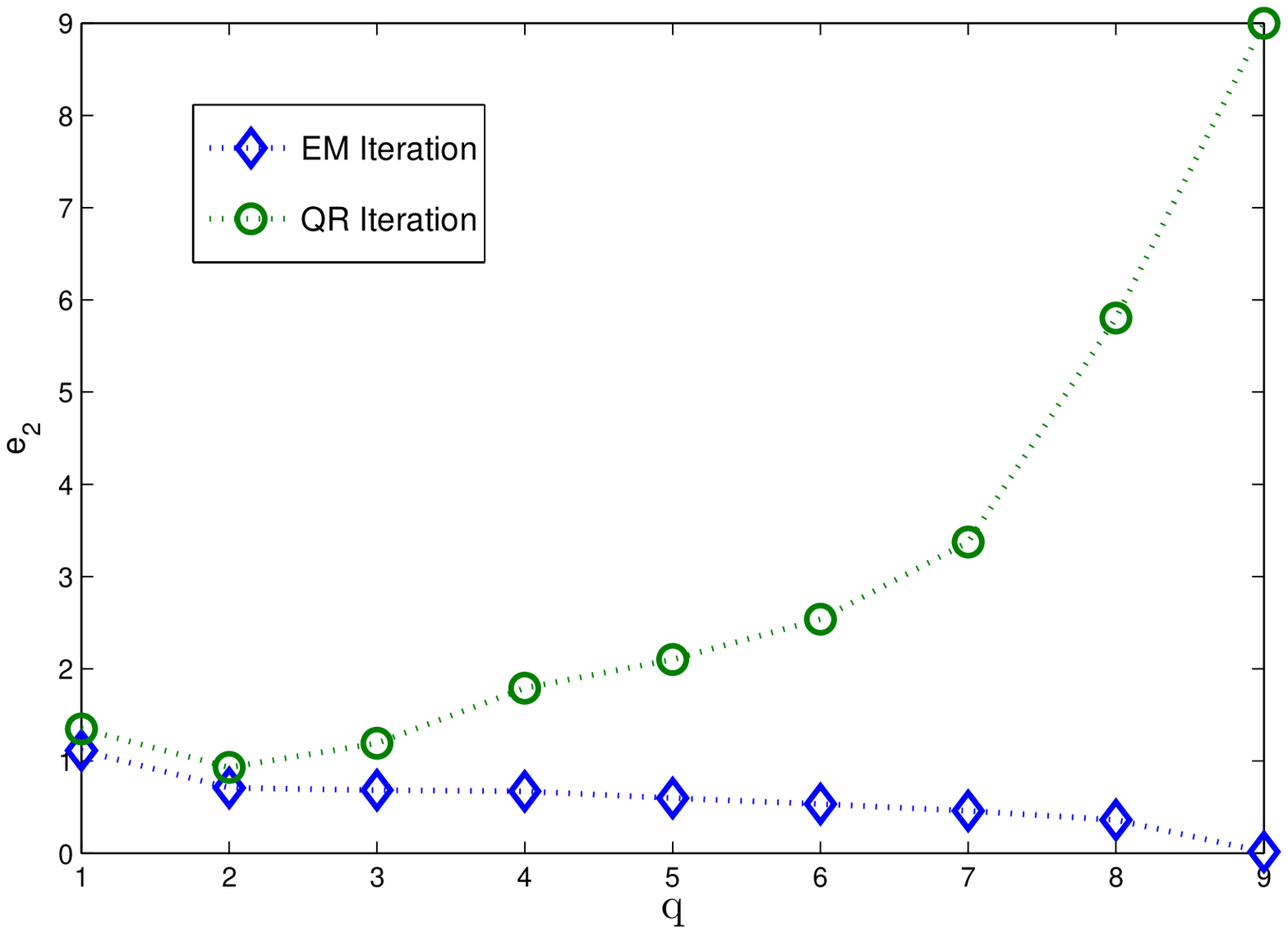}\\
(a) \; $e_F$  vs. $q$  & (b) \;  $e_2$ vs. $q$  
\\
\end{tabular}
\caption{(a) The errors between $\M^{-1}$ and its
approximate $\A \A^T+\delta \I_m$ where $\A$ and $\delta$ were
estimated by the EM method and the QR method respectively, for $q=1, \ldots, 9$.
}
\label{fig:exam1}
\end{figure}

Let us see the estimates of $\A$ in the cases that $q=1$ and $q=3$.
First, when $q=1$, the EM estimate of $\A$ is
\begin{small}
\[
\A= (-0.9563, -0.9790,  -0.9126,  -0.9774, -0.9308, -0.6513,
-0.9108,  -0.9579,  -0.8809, -1.0007)^T.
\]
\end{small}
It is further seen that $\A(\A^T\A)^{-\frac{1}{2}}=(-0.3285, -0.3363,
-0.3135,  -0.3357, -0.3197, -0.2237, \\ -0.3128, -0.3290, -0.3026,
-0.3437)^T$ is the principal eigenvector of $\M$.

When $q=3$, the matrix $\U_3$  of the first three eigenvectors of
$\M$ and the EM estimate of $\A(\A^T\A)^{-\frac{1}{2}}$ are
respectively given by
\begin{small}
\[
\U_3 = 
 \begin{bmatrix}
 -0.3285  &  0.4057  &  0.1792 \\
   -0.3363 &  -0.1540 &  -0.4530 \\
   -0.3135  & -0.0746 &   0.0302 \\
   -0.3357 &  -0.3073 &   0.1697 \\
   -0.3197  &  0.3362 &   0.1897 \\
   -0.2237 &   0.4044 &   0.1239 \\
   -0.3128 &  -0.0221 &   0.1241 \\
   -0.3290 &   0.2035 &   0.0465 \\
   -0.3026 &  -0.6230 &   0.4150 \\
   -0.3437 &  -0.0711 &  -0.7013 \end{bmatrix} 
 \;  \mbox{ and } \;
\A(\A^T\A)^{-\frac{1}{2}} = 
\begin{bmatrix}
0.1658  & -0.3889  & -0.3549 \\
0.3884  & -0.1290 &   0.4159 \\
0.0359  & -0.2973  &  0.1226 \\
-0.1863 &  -0.3615 &   0.2647 \\
0.1194  & -0.3819  & -0.3019 \\
0.1769  & -0.2709  & -0.3525 \\
-0.0128 &  -0.3352 &   0.0344 \\
0.1695  & -0.3296  & -0.1203 \\
-0.5560 &  -0.4134 &   0.4152 \\
0.6341  & -0.0423  &  0.4610  \end{bmatrix} 
\]
\end{small}
It is easily verified that
\[
\A(\A^T\A)^{-\frac{1}{2}} = \U_3 \V
\]
where
\[
\V= \left[ \begin{array}{ccc} -0.3130 &   0.9214 &  -0.2304 \\
    0.5098 &  -0.0417 &  -0.8593 \\
   -0.8013 &  -0.3864 &  -0.4566  \end{array} \right]
\]
is a $3{\times}3$ orthogonal matrix. This is in line with the
theoretical justification given in Section~\ref{sec:app}.

In the case that $q=3$, we also implement the incomplete Cholesky decomposition of $\M$ as $\M \approx \L \L^T$
where $\L$ and $\L (\L^T\L)^{-\frac{1}{2}}$ are given as\footnote{
Our implementation is based on the code from {\tt http://theoval.cmp.uea.ac.uk/~gcc/matlab/default.html}, which was written
for the incomplete Cholesky decomposition algorithm described by \citet{Fine:2001}.}
\[ \L  = \begin{bmatrix} 0.6391 &   0.2092    &       1.16727 \\
 0.7028     &   0.3565   &       0.2924 \\
 0.6786     &       0.3931       &       0.3103 \\
 0.6786     &       0.4302       &       0.2458 \\
 0.6202     &       0.2184       &       0.4694 \\
 0.4284     &       0.0659      &       0.4621 \\
 0.6359     &       0.3751       &       0.3331 \\
 0.6752     &       0.2549       &       0.4286 \\
 0.6110     &       1.2587        &       0 \\
 1.4017     &       0                       &   0
\end{bmatrix} \;
\mbox{ and } \;
 \L(\L^T\L)^{-\frac{1}{2}} = \begin{bmatrix}
 0.0633      &       0.0016     &       0.8150 \\
 0.2621      &       0.1315     &       0.0757 \\
 0.2365      &       0.1659     &       0.0948 \\
 0.2428      &       0.1977     &       0.0407 \\
 0.2051      &       0.0321     &       0.2437 \\
 0.1292      &      -0.0505     &       0.2843 \\
 0.2108      &       0.1598     &       0.1236 \\
 0.2379      &       0.0512     &       0.1968 \\
 0.0816      &       0.8915     &      -0.1696 \\
 0.8037      &      -0.2992     &      -0.3104
\end{bmatrix}.
\]
Assume that $\ran(\L(\L^T\L)^{-\frac{1}{2}})=\ran(\U_3)$. Then we have $\L(\L^T\L)^{-\frac{1}{2}} = \U_3 \R$ where
\[
\R =\U_3^T \L(\L^T\L)^{-\frac{1}{2}} = \begin{bmatrix} 0.6248   &       -0.5466      &       0.4554 \\
   -0.6634   &       -0.5633       &       0.3665  \\
  0.16903    &       -0.5628       &       -0.6421
\end{bmatrix}.
\]
We further have $\I_3=(\L^T\L)^{-\frac{1}{2}} \L^T\L (\L^T\L)^{-\frac{1}{2}} = \R^T \U_3^T\U_3 \R = \R^T\R$.
However, it is directly computed that $\R^T\R \neq \I_3$, yielding a conflict. This implies that the assumption
$\ran(\L(\L^T\L)^{-\frac{1}{2}})=\ran(\U_3)$ is not true. Thus, this example shows that the incomplete
Cholesky decomposition can not be used to find the top eigenvectors of an arbitrary positive definite matrix.

Additionally, we defined a new positive definite matrix $\K$ as
\[
\K = \M + \alpha^2 \I_{10},
\]
which has an  explicit form as in (\ref{eqn:spe_str}). As mentioned in Section~\ref{sec:app},
we  employed the incomplete Cholesky decomposition to approximate $\K^{-1}$. In particular, we first obtained
the incomplete Cholesky decomposition of $\M$ as $\M \approx  \L \L^T$ and then computed $(\alpha^2 \I_{10} + \L \L^T)^{-1}
=\alpha^{-2} \I_{10}- \alpha^{-2} \L(\alpha^2 \I_q + \L^T \L)^{-1} \L^T$ as the approximation to $\K^{-1}$.
Let $q=3$. We took $\alpha^2 =0.1$ and $\alpha^2 =0.0001$ to implement the empirical analysis.
When   $\alpha^2 =0.1$,  $e_F$ and $e_2$ with the incomplete Cholesky decomposition are  respectively
$7.0688$ and $14.8785$;  $e_F$ and $e_2$ with our method are  $0.3030$ and  $0.5886$.
When $\alpha^2 =0.0001$,  $e_F$ and $e_2$ with the incomplete Cholesky decomposition are  respectively
$7.0680{\times}10^3$ and $1.4877{\times} 10^4$;  $e_F$ and $e_2$ with our method are  $0.3522$ and  $ 0.6840$.
This shows that the incomplete Cholesky decomposition fails when $\alpha^2$ takes a very small value.
However, our method is numerically stable in every case. The reason is in that
our method makes $\A \A^T {+} \delta \I_{10}$ better-conditioned  than $\K$. But we see that $\L \L^T {+} \alpha^2 \I_{10}$
is more ill-conditioned  than $\K$.

Finally, we performed a simulation on a cluster
to further validate efficiency of our approach in inverting large-size matrices.
We randomly generated a $50000  \times 50000$ positive definite matrix $\M$ from Wishart distribution $W_{50000}(50020, \Si)$
where $\Si= 0.5 \1_{50000} \1_{50000}^T + 0.5 \I_{50000}$. The running time of the direct computation for $\M^{-1}$
is $5.4416 \times 10^4 $ (s), while our approximate approach with $q=224$ ($\approx \sqrt{50000}$)  took  $2.5452{\times} 10^3$
seconds. Moreover, the errors are $e_f= 0.9744$ and $e_2=2.4798$, respectively.

\subsection{The Matrix Ridge Approximation for Spectral Clustering}
\label{sec:exp2}

The matrix ridge approximation (RA)  with the EM iteration has potentially wide applications in those
methods who involve the inversion or SD of a large-scale positive semidefinite matrix.
In this section we apply RA to spectral clustering.

Spectral clustering~\citep{Shi:2000,Ng:2002} is a method for partitioning data into classes
by relaxing an intractable partitioning problem into a tractable
eigenvector problem, specifically a problem that can be reduced
to the eigenvector problem in~(\ref{eq:sep}) for a particular
matrix $\M$~\citep{ZhangzhSTS:2008a}.  The solution of the relaxation is then ``rounded''
to yield a partition, where standard rounding methods include
$K$-means and Procrustes analysis~\citep{ZhangzhSTS:2008a}.

In the following experiments, we used the  EM-based RA  methods to solve the eigenvector
relaxation associated with spectral clustering, and compared the
results with the conventional direct spectral decomposition (SD) method.  We also
implemented the $K$-means and Procrustean transformation (PT)
rounding algorithms given in~\cite{ZhangzhSTS:2008a} to obtain
complete spectral clustering algorithms.  This yields four spectral
clustering algorithms, which we refer to as RA-KM, SD-KM,
RA-PT and SD-PT.

Assume we are given a dataset $\{\x_1, \ldots, \x_m\}$.
We defined $\M$ as a kernel matrix $\K$
via the RBF kernel with single parameter $\beta$, i.e., $[\K]_{ij}=K(\x_i, \x_j)=\exp ({-} {\|\x_{i} - \x_{j}
\|^2}/{\beta})$. Let $\BP=\I_m {-} \frac{1}{m}\1_m \1_m^T$ (i.e., $\b=\frac{1}{\sqrt{m}} \1_m$). We then formed the $m{\times}m$
matrix $\T=\BP \K \BP$, whose top $q$ eigenvectors are the solution of the eigenvector problem in~(\ref{eq:sep}).
That is, the top $q$ eigenvectors of $\T$ are just the eigenvector
relaxation associated with spectral clustering. Recall that $\hat{\A} (\hat{\A}^T
\hat{\A})^{-\frac{1}{2}}= \U_q \V = \widehat{\X}$, which implies that
RA-KM and RA-PT employ the EM-based RA to find such $q$ eigenvectors. However,
SD-KM and SD-PT employ the standard direct SD to obtain the $q$ eigenvectors.

We conducted the experiments on eight publicly
available datasets from the UCI Machine Learning Repository: the
{\tt dermatology} data, the {\tt soybean} data, the ``A-J"
{\tt letter} data,  the {\tt image segmentation} data, the
{\tt NIST} optical handwritten digit data, the {\tt CTG} (Cardiotocograms) data,
the {\tt pen-based recognition of handwritten digits}  data,
and the {\tt Statlog} (Landsat Satellite) data. Table~\ref{tab:data}
gives a summary of these datasets.

In the clustering setup, $q{+}1$ is the number of classes. We initialized $K$-means by the orthogonal
initialization method in \cite{Ng:2002} and the Procrustean
transformation by $\I_{q}$. The values of $\beta$ that we used are given in the last row of
Table~\ref{tab:data1}; they were set to be empirically optimal for these
algorithms.

\begin{table}[!ht]
\begin{center}
\caption{Summary of the benchmark datasets: $m$---\# of
samples; $p$---\# of features; $q{+}1$---\# of classes; $\beta$---parameter in the kernel function $K(\cdot, \cdot)$.}
\label{tab:data1}
\begin{tabular}{|l|cccccccc|} \hline
     & {\tt Derma} & {\tt Soybean} & {\tt Letter} & {\tt CTG} & {\tt Segmen} & {\tt NIST}  & {\tt Landsat}  & {\tt Pen} \\
\hline
$m$    & $358$     & $630$   & $1978$ & $2126$ & $2310$  & $3823$   & $4435$    & $7494$  \\
$p$    & $34$      & $35$    & $16$   & $23$   & $18$    & $59$     & $36$     & $16$ \\
$q{+}1$    & $6$       & $19$    & $10$    & $10$    & $7$     & $10$      & $7$     & $10$  \\
\hline
$\beta$ & $100$     & $100$  & $100$  & $100$  & $1000$  & $1000$  & $5000$      & $100$ \\
\hline
\end{tabular}
\end{center}
\end{table}

To evaluate the performance of the various clustering algorithms,
we employed the Rand index (RI)~\citep{RandRI:1971}. Given a set
of $m$ samples ${\cal X} = \{\x_1, \ldots, \x_m\}$, suppose that
${\cal U} = \{{\cal U}_1, \ldots, {\cal U}_r\}$ and ${\cal V} = \{{\cal V}_1, \ldots, {\cal V}_s\}$ are
two different partitions of the samples in ${\cal X}$ such that
$\cup_{i=1}^r {\cal U}_i = {\cal X} = \cup_{j=1}^s {\cal V}_j$ and ${\cal U}_i
\cap {\cal U}_{i'} = \emptyset = {\cal V}_j \cap {\cal V}_{j'}$ for $i\neq i'$ and $j
\neq j'$. Let $a$ be the number of pairs of samples that are in the
same set in ${\cal U}$ and in the same set in ${\cal V}$, and $b$ the number of
pairs of samples that are in different sets in ${\cal U}$ and in different
sets in ${\cal V}$. The RI is given by $\mathrm{RI} =  (a+b)/{m
\choose 2 }$. If $\mathrm{RI}=1$, the two partitions are
identical.
Since the
true partitions are available for our datasets, we calculated
the RI between the true partition and the partition obtained
from each clustering algorithm.

We conducted 50 replicates of
each of those algorithms with $K$-means rounding because
of the random initialization required by $K$-means (this is
not necessary for the Procrustean transformation, because it
is initialized to the identity matrix). The results shown in
Table~\ref{tab:ri} are based on the average of these 50
realizations.

From Table~\ref{tab:ri} we see that the clustering methods based on
RA and SD have the almost same clustering performance.
In Table~\ref{tab:cputime} we reported the CPU times of the
direct SD method and the EM-based RA method for computing the top $q$ eigenvectors. We
see that RA method can be significantly more
efficient than the SD method for large $m$, and this is borne out
by our results. For example,
on \emph{pen-based recognition} of handwritten digits data ($m=7494$), the direct
SD method takes twenty two minutes, the EM-based RA method only needs about four minutes.

\begin{table}[!ht]
\begin{center}
\caption{Rand Index (\%).}
\label{tab:ri}
\begin{tabular}{|l||c|c||c|c|c|} \hline
 &   SD-PT  &   RA-PT  &  SD-KM & RA-KM   \\
\hline
{\tt Derma}     & ${95.49}$ &  ${95.49}$ &  $94.57$ ($\pm 1.89$) &  $94.47$ ($\pm 3.41$)  \\
  {\tt Soybean} & $92.69$  & ${92.87}$   & $91.32$ ($\pm 1.22$) & $91.80$ ($\pm 0.98$)  \\
 {\tt Letter}   & $ 85.68$ & ${85.63}$   & $84.96$ ($\pm 0.49$) & $84.96$ ($\pm 0.43$) \\
 {\tt CTG}      & $ 85.68$ & ${85.63}$   & $84.96$ ($\pm 0.49$) & $84.96$ ($\pm 0.43$) \\
{\tt Segmen}    & $ 80.51$  & $81.32$    & ${75.05}$ ($\pm 3.3$) & ${79.00}$($\pm 2.27$) \\
{\tt NIST}      & ${ 89.90}$  & ${89.89}$  & $89.52$ ($\pm 0.70$)  &  $89.51$ ($\pm 0.63$)  \\
{\tt Landsat}   & ${84.88}$  & ${84.90}$  & $83.33$ ($\pm 0.63$)  &  $83.28$ ($\pm 0.66$)\\
{\tt Pen}       & $90.61$  & $90.64$      & ${91.14}$ ($\pm 0.48$) & ${91.07}$($\pm 0.52$) \\
 \hline
\end{tabular}
\end{center}
\end{table}

\begin{table}[!t]
\begin{center}
\caption{CPU times (s) of running the spectral relaxation with
the direct SD and  EM-based ridge approximation (RA) which are performed in Matlab on a
Core 2 Duo computer with a 2.27 GHz CPU and 8 GB of RAM.} \label{tab:cputime}
\begin{tabular}{|l|cccccccc|} \hline
          &  {\tt Derma} & {\tt Soybean}  & {\tt Letter} & {\tt CTG} &{\tt Segmen} & {\tt NIST} & {\tt Landsat} & {\tt Pen} \\ \hline
{\tt SD} & $0.2202$  & $0.9148$ & $28.3251$ & $32.9942$ & $35.2645$ & $212.1458$ & $289.9903$ & $1361.8$  \\
\hline
{\tt RA} &  $0.1862$ & $0.7485$ & $7.8297$ & $9.4303$  & $11.1811$ & $42.7599$ & $64.054$  & $266.9177$ \\
\hline
\end{tabular}
\end{center}
\end{table}

\subsection{The Matrix Ridge Approximation for GPR}
\label{sec:exp3}

In this section we applied the matrix ridge approximation with the EM iteration to Gaussian process regression (GPR), and compared with
the Nystr\"{o}m method~\citep{WilliamsNIPS:2001_2}
and the incomplete Cholesky decomposition method~\citep{Fine:2001}.

Assume we are given a training dataset  ${\mathcal D} = \{(\x_1, y_1), \ldots, (\x_m, y_m)\}$, where the
$\x_i \in \BR^p$ are the input vectors and $y_i \in \BR$ are the corresponding outputs. In the GPR model $y$ is defined as
\[
y = u + f(\x) + \epsilon,  \quad \epsilon \thicksim N(0, \sigma^2),
\]
where $f(\x)$ follows a Gaussian process with mean function 0 and covariance function $K(\cdot, \cdot)$.
This implies that $\f = (f(\x_1), f(\x_2),
\ldots, f(\x_m))^T$, corresponding outputs of the input vectors in the
training dataset ${\mathcal D}$, has multivariate Gaussian distribution $N(\0, \K)$, where $\K$ is the $m{\times}m$ covariance matrix with
$[\K]_{ij} = K(\x_i, \x_j)$.

We employed the Gaussian RBF kernel function $K(\cdot, \cdot)$ with a separate length-scale
parameter for each variate of the input vector, plus the signal and noise variance
parameters $\sigma_f^2$ and $\sigma^2$. These parameters are trained by optimizing the marginal
likelihood on a subset of the training data. Here we ignored the learning details and directly used the code provided
by \cite{RassmussenWilliams} to implement the training. 
We concentrated our attention on the test procure.

For a test input vector $\x_{*}$,  the prediction of the
corresponding output $y_{*}$ is based on the conditional posterior distribution $p(y_{*}|\y)$, which is also Gaussian.
In particular, the predicted mean at $\x_*$ is given by
\[
\hat{y}_{*} = {\bf k}^T(\x_*) (\K {+} \sigma^2 \I_m)^{-1} {\bf
y},
\] 
where $\y = (y_1, \ldots, y_m)^T$
 and ${\bf k}(\x_*) = ( K(\x_*, \x_1),
\ldots, K(\x_*, \x_m))^T$ \citep[see,][]{RassmussenWilliams}. As we
can see, GPR requires us to compute the inverse of $ \K
{+} \sigma^2 \I_m$, which is an $m{\times}m$ positive definite matrix. When the size ($m$) of the training dataset is very large,
this limits the efficient application of
GPR.

Since $\K {+} \sigma^2 \I_m$ has the explicit form mentioned in Section~\ref{sec:app},
\citet{WilliamsNIPS:2001_2} considered the Nystr\"{o}m
approximation for its inverse when $m$ is large. The Nystr\"{o}m method randomly chooses $q$ columns of $\K$
without replacement. Let $\K_{m, q}$ denote the $m{\times} q$ matrix consisting of such $q$ columns. Then the Nystr\"{o}m
approximation of $\K$ is $\K_{m, q} \K_{q, q}^{-1} \K_{m, q}^T$.
Here we also compared the approximate method based on the incomplete Cholesky
decomposition; that is, we first implemented the incomplete Cholesky
decomposition of $\K$ as $\K \approx \L \L^T$ where $\L$ is an $m{\times}q$ lower triangular matrix.
After having obtained the Nystr\"{o}m approximation or the incomplete Cholesky
decomposition, we then computed $(\K_{m, q} \K_{q, q}^{-1} \K_{m, q}^T {+} \sigma^2 \I_m)^{-1}$ or $(\L \L^T {+} \sigma^2)^{-1} \I_m$ via
the Sherman-Morrison-Woodbury formula. Recall that we applied the ridge approximation directly on $\K {+} \sigma^2 \I_m$,
rather than on $\K$.

We conducted the experiments on seven publicly
available datasets from the UCI Machine Learning Repository: the {\tt Boston Housing} data,  the
{\tt Concrete Compressive Strength } (CCS) data, the {\tt Contraceptive Method Choice} (CMC) data, the
{\tt Abalone} data, the {\tt Landsat Satellite}  (Sat) data, the {\tt SARCOS} data, and the {\tt YearPredictionMSD} ({YPMSD}) data. We employed the setting given
in the UCI Machine Learning Repository for training and testing for the first six datasets. For the {\tt YPMSD} data, we employed two settings for training and testing. In the first setting ({\tt YPMSD1}) we used
the first $60,000$ samples for training and the rest of the samples for testing, while in the second setting ({\tt YPMSD2}) we used
the first $100,000$ samples for training and the rest of the samples for testing. Table~\ref{tab:data}
gives a summary of these datasets.

\begin{table}[!t]
\begin{center}
\caption{Summary of the  datasets: $m$---\# of
training samples; $n$---\# of test samples; $p$---\# of features;}
\label{tab:data}
\begin{tabular}{|l|cccccccc|} \hline
    & {\tt Housing} & {\tt CCS} & {\tt CMC} &  {\tt Abalone}  &  {\tt Sat} & {\tt  SARCOS} & {\tt YPMSD1} & {\tt YPMSD2}  \\
\hline
$m$   & $455$        & $700$   & $1,000$   & $3,133$ & $4,435$ & $5,000$ &60,000 & $100,000$ \\
$n$   & $51$         & $330$   & $473$    & $1,044$ & $2,000$ & $4,449$  &455,345  & 415,3455  \\
\hline
$p$   & $13$         & $8$     & $9$      & $8$      & $36$   & $21$ & 90 & $90$  \\
\hline
\end{tabular}
\end{center}
\end{table}

We evaluated the
performance of predictions using the standardized mean squared error (SMSE)~\citep{RassmussenWilliams}.
We set the rank  of the  matrix $\A$ in the ridge approximation, the columns of the matrix $\L$($\K \approx \L\L^T$)
in the incomplete Cholesky decomposition,
and the columns  uniformly sampled from the original kernel matrix in the Nystr\"{o}m method to the same number $q$.
We then compared the performance of the three methods.

For the Nystr\"{o}m method, for each given $q$,
we repeated the experiment 50 times. We found that the results are very sensitive to the columns randomly selected.
The method works well in a few instances, but in
most cases its performance is extremely poor. Thus, given $q$, we reported the smallest SMSE for the Nystr\"{o}m method.

Figure~\ref{fig:GR:subfig} shows  SMSE values over the first six datasets. It should be worth pointing out that
the performance of the Nystr\"{o}m method is very poor on the Sat  and CCS datasets.
Thus, we omitted the SMSE values  on the two datasets for the  Nystr\"{o}m method. Also,
when $q$ is less than 4396 for the Sat dataset, the performance of
the incomplete Cholesky decomposition method  is poor. For this reason,
we also omitted the SMSE values for the incomplete Cholesky decomposition method on the Sat dataset.

\begin{figure} 
\vspace{-0.1in}
\centering
\subfigure[Housing Data]{
\label{fig:GR:subfig:e}
\begin{minipage}[b]{0.45\textwidth}
\centering
\includegraphics[width=65mm,height=56mm]{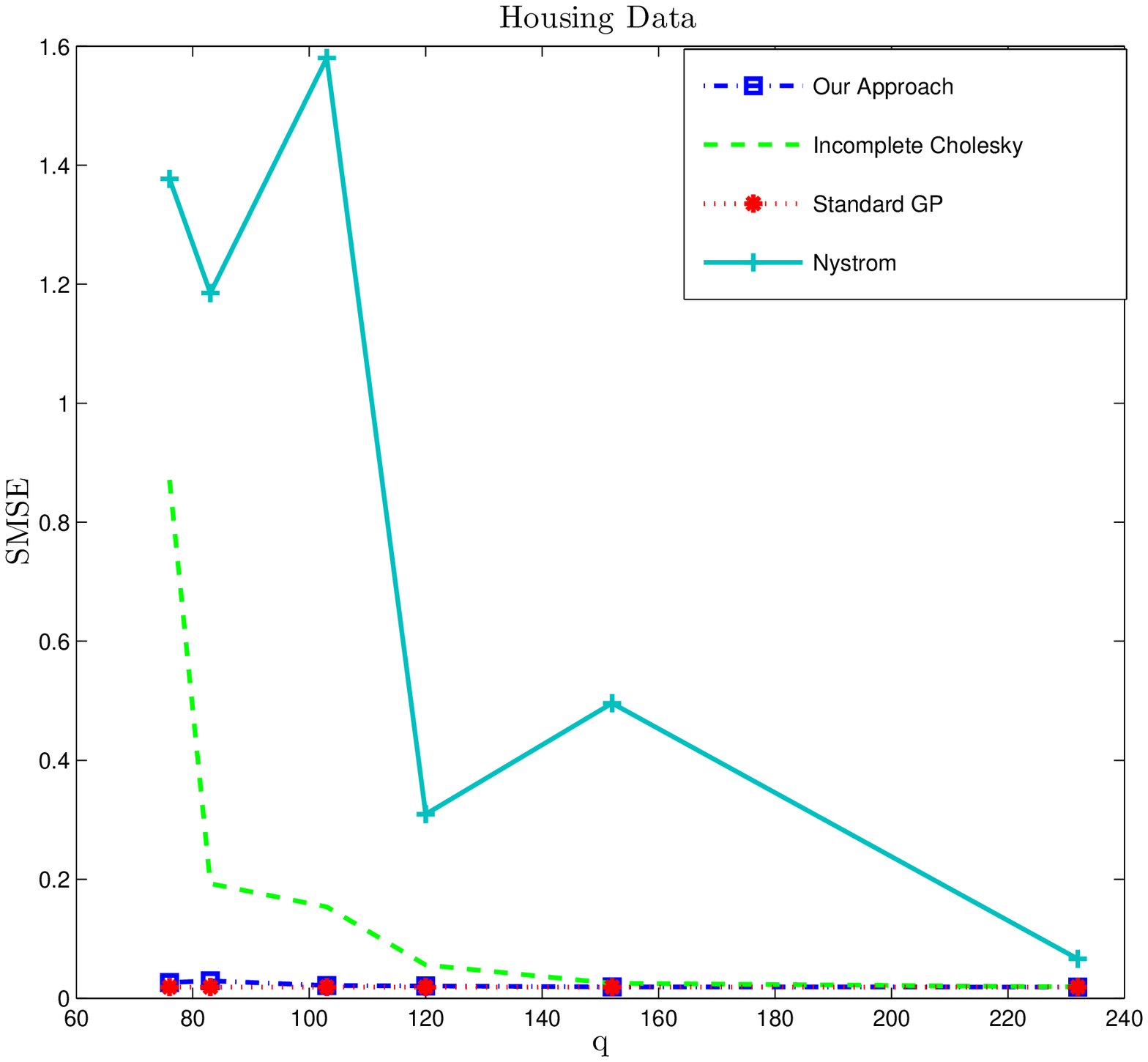}
\end{minipage}
}
\subfigure[Concrete Compressive Strength Data]{
\label{fig:GR:subfig:c}
\begin{minipage}[b]{0.45\textwidth}
\centering
\includegraphics[width=65mm,height=56mm]{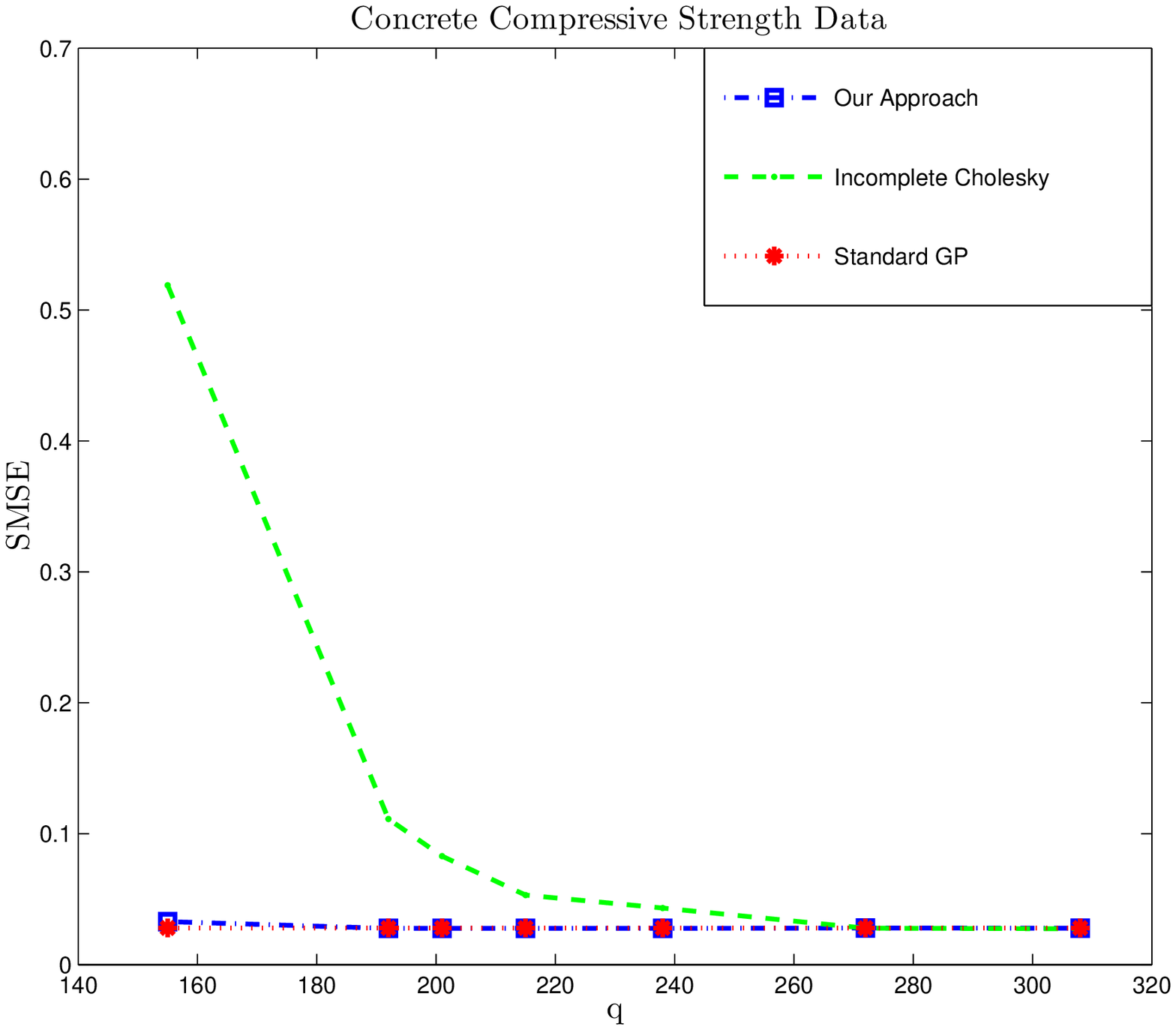}
\end{minipage}
}
\subfigure[Contraceptive Method Choice Data]{
\label{fig:GR:subfig:d}
\begin{minipage}[b]{0.45\textwidth}
\centering
\includegraphics[width=65mm,height=56mm]{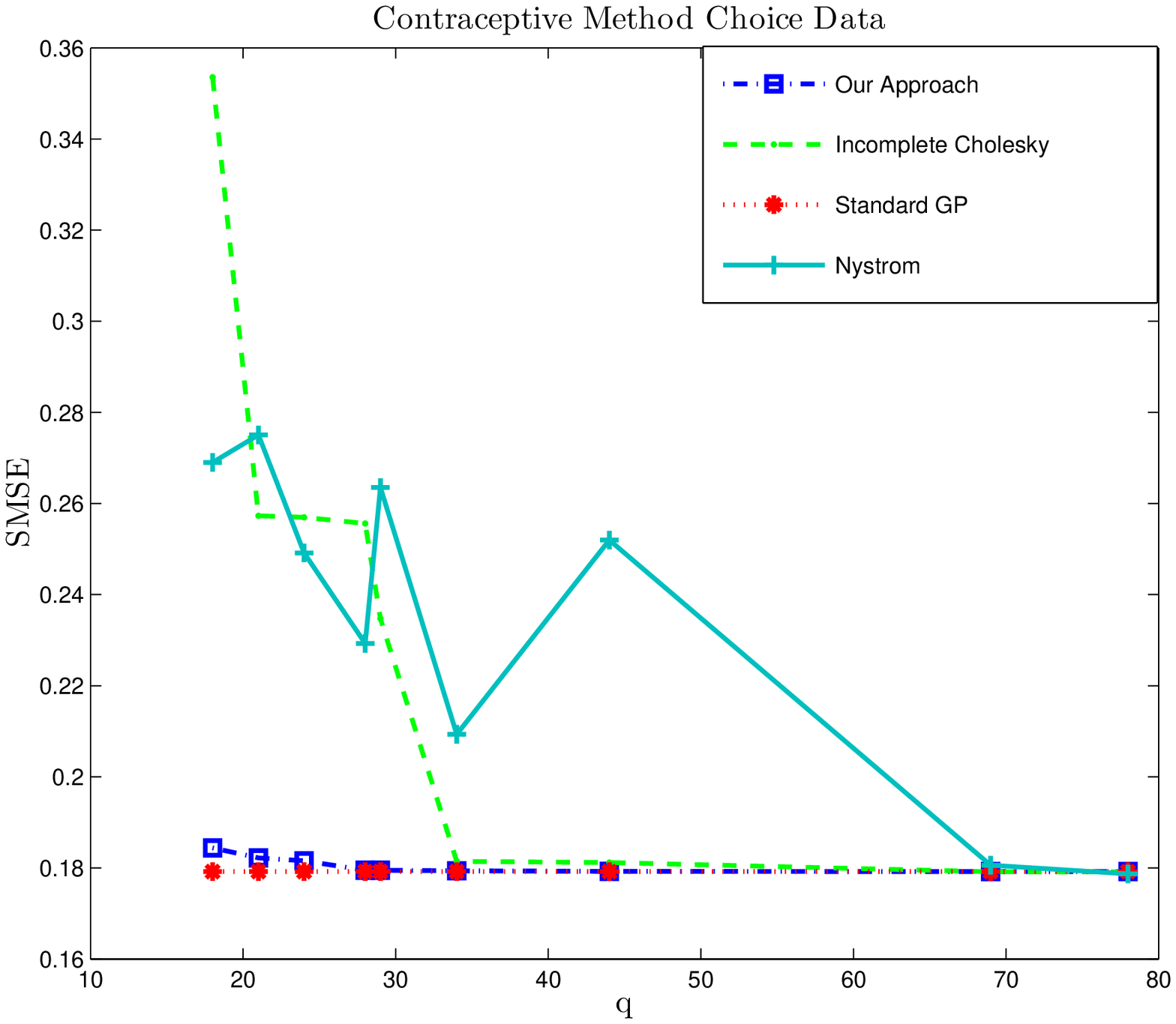}
\end{minipage}
}
\subfigure[Abalone Data]{
\label{fig:GR:subfig:a}
\begin{minipage}[b]{0.45\textwidth}
\centering
\includegraphics[width=65mm,height=56mm]{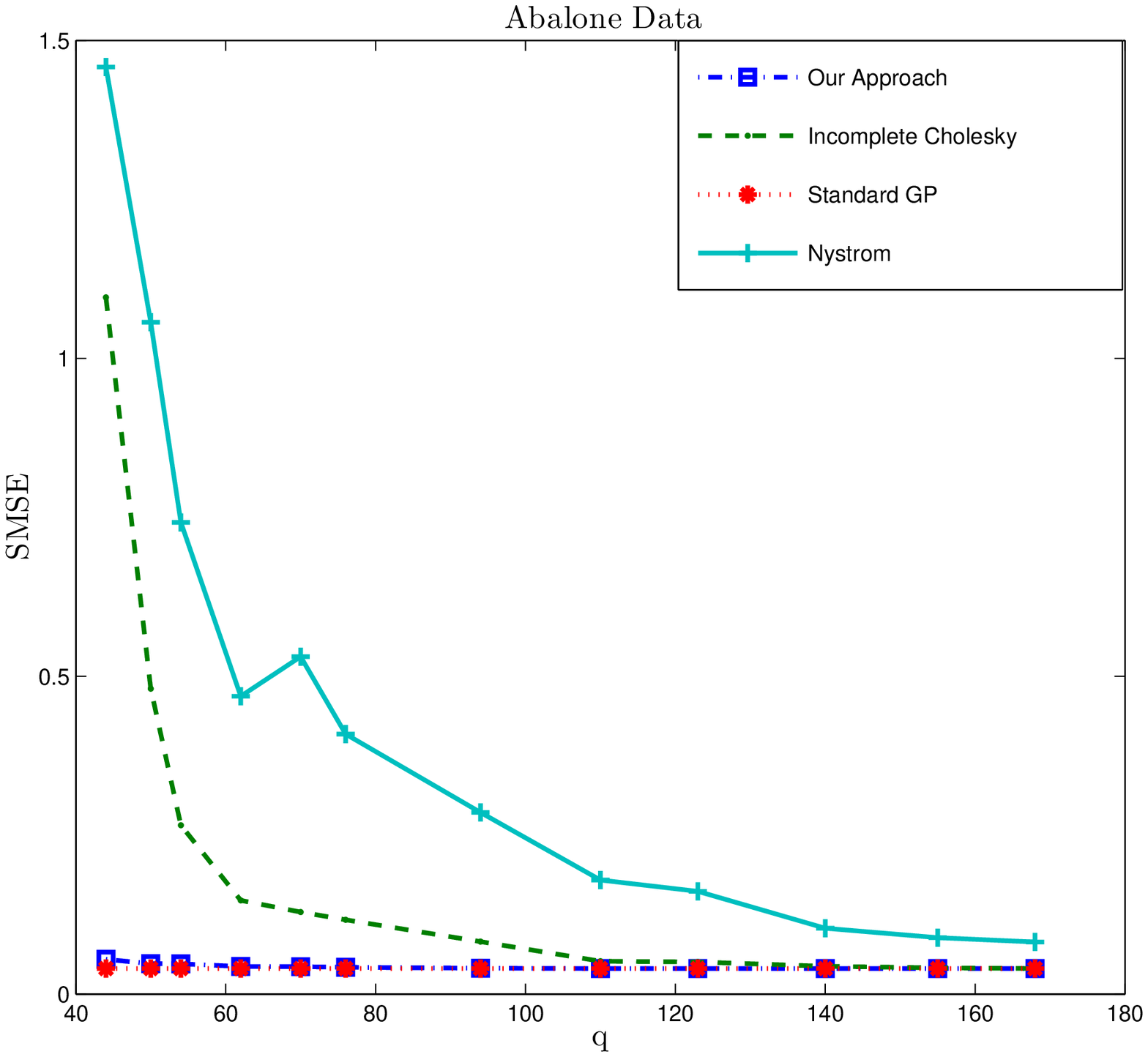}
\end{minipage}
}
\subfigure[Landsat Satellite Data]{
\label{fig:GR:subfig:b}
\begin{minipage}[b]{0.45\textwidth}
\centering
\includegraphics[width=65mm,height=56mm]{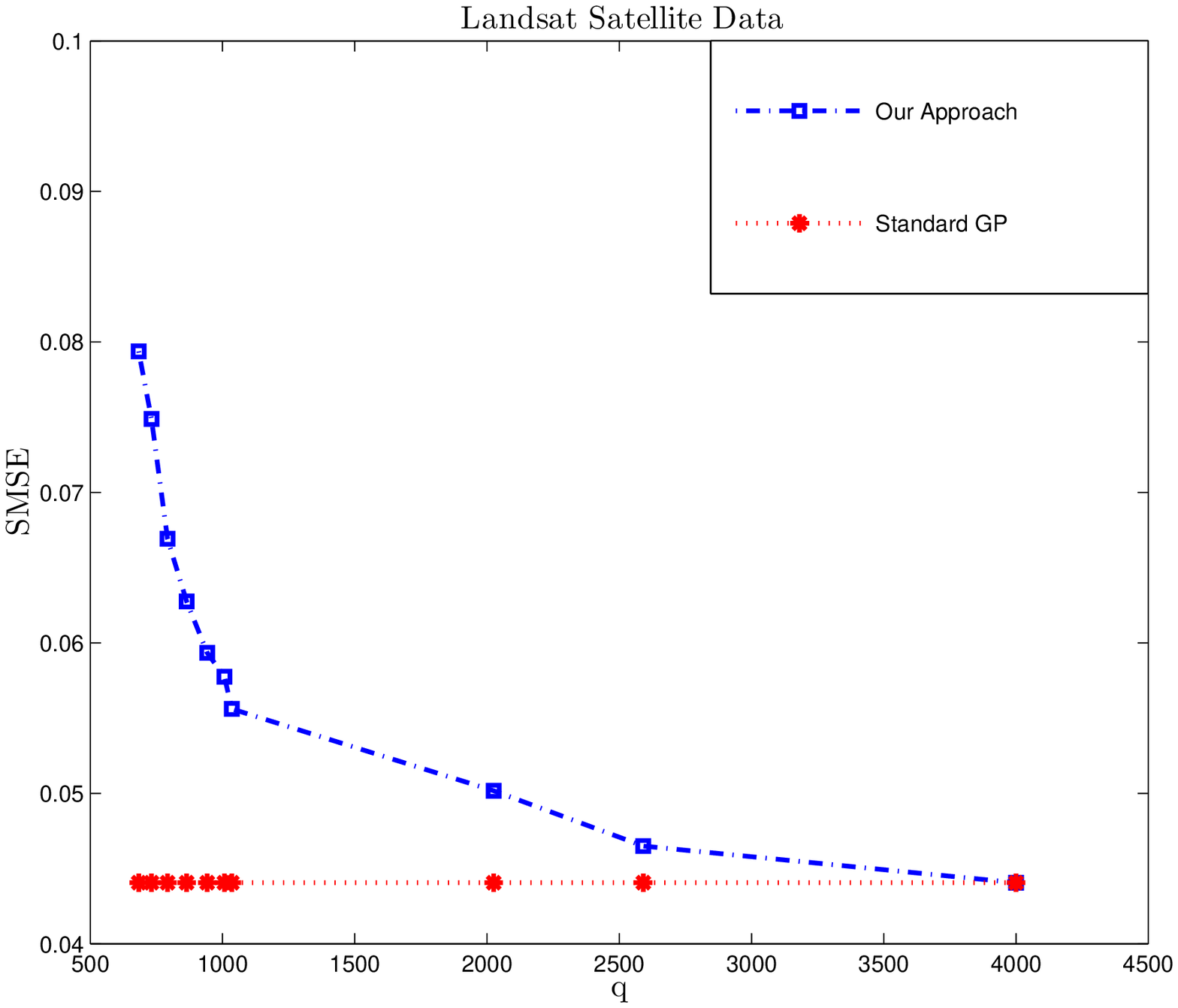}
\end{minipage}
}
\subfigure[SARCOS Data]{
\label{fig:GR:subfig:f}
\begin{minipage}[b]{0.45\textwidth}
\centering
\includegraphics[width=65mm,height=56mm]{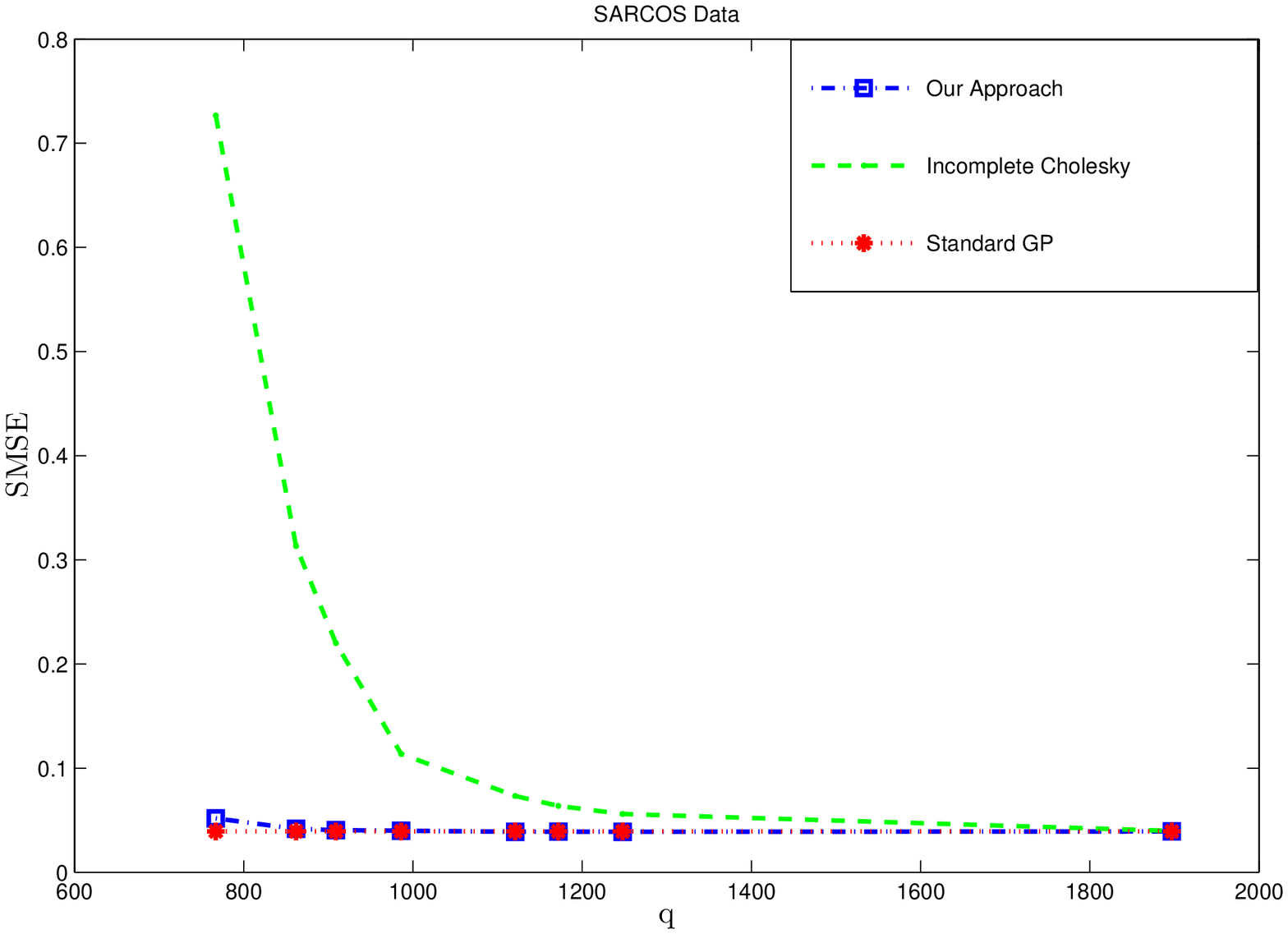}
\end{minipage}
}
\caption{Comparisons of  the ridge approximation method, the
Nystr\"{o}m method and the incomplete Cholesky decomposition method. }
\label{fig:GR:subfig}
\end{figure}

From Figure~\ref{fig:GR:subfig}, we see that the performance of the ridge approximation method
is nearly the same as that of standard GPR.
Moreover, the ridge approximation is not sensitive to the value of $q$.
For a wide range of $q$,  the GPR prediction varies very little. When $q$ takes a small value, the ridge approximation still works well.
Contrarily,  when $q$ is small, the incomplete Cholesky decomposition is not very
effective, because it results in an underfitting problem in which
$\L \L^T {+} \sigma^2 \I_{m}$
is ill-conditioned.
However, the ridge approximation can avoid this
problem, because it makes $\A \A^T {+} \delta \I_{m}$ better-conditioned  than $\K$ itself (see the discussion in Section~\ref{sec:exp1}).

For the {\tt YPMSD} data,  we did not include the results with the
Nystr\"{o}m method and the incomplete Cholesky decomposition method because the performance of these two methods is very poor when $q\ll m$ (e.g., $q\leq \sqrt{m}$).
We only took $q=245$ for {\tt YPMSD1} and $q=316$ for {\tt YPMSD2} ($\approx \sqrt{m}$) to implement the ridge approximation method.
Since the size of the matrix $\K + \sigma^2 \I_m$ is too large, we partitioned it into the $2 {\times}2$
block submatrices in the direct computation of $(\K + \sigma^2 \I_m)^{-1}$. Although this does not reduce the computational cost,  it can make the  computation more numerically stable.  To reduce the storage space of  data,
all computations were carried out with  single precision in Matlab.
However, we still could not complete  the experiment with the direct computation method on the
{\tt YPMSD2} dataset due to  limited  storage space.

The SMSE values with the direct method and the ridge approximation
method for computing $(\K + \sigma^2 \I_m)^{-1}$ on {\tt YPMSD1}
are $3.7713{\times} 10^{-5}$ and $2.9232{\times} 10^{-5}$, respectively.
We see that the ridge approximation slightly outperforms the direct computation. We hypothesize that this phenomenon
is a result of  roundoff error in the floating point computations.
The SMSE value with the ridge approximation on {\tt YPMSD2} is  $2.8959{\times} 10^{-5}$.
Therefore, the ridge approximation method is effective.

Finally, in Table~\ref{tab:cputime2} we  report the running times with our matrix ridge approximation
and the direct calculation for $(\K+\sigma^2 \I_m)^{-1}$ on the  datasets.
The reported results with our method are based on that $q$ is taken as the integer closest to $\sqrt{m}$.
We see that our method is able to reduce computation when $m$ is vary large. For example, on the {\tt YPMSD1} dataset the direct computation  took
$1.2436 {\times} 10^5$ seconds, while the ridge approximation took $5.814 {\times} 10^{3}$ seconds.
In summary,  our proposed  approach is efficient and effective.

\begin{table}[!ht]
\begin{center}
\caption{CPU times (s) of running the test procedure of GPR with
the direct computation and  EM-based ridge approximation (RA) which are performed in Matlab on a Workstation
with a 3.07 GHz CPU and 24 GB of RAM.} \label{tab:cputime2}
\begin{tabular}{|l|cccccccc|} \hline
    &  {\tt Housing} &  {\tt CCS} & {\tt CMC} & {\tt Abalone}  &   {\tt Sat}   & {\tt  SARCOS} & {\tt  YPMSD1}  & {\tt  YPMSD2}  \\
    \hline
{\tt Direct} & $0.7020$  & $0.5460$  &    $0.7644$       & $33.8210$   & $105.0823$ & $145.8141$ & $1.2436 {\times} 10^5$ & NA \\
\hline
{\tt EM-RA} &  $0.1872$ & $0.530$ & $ 0.9048$  & $9.8437$  & $18.7981$ & $20.8729$  & $5.814 {\times} 10^{3}$ & $1.0664 {\times} 10^{5}$ \\
\hline
\end{tabular}
\end{center}
\end{table}

\section{Conclusion} \label{sec:conclusion}

In this paper we have proposed the matrix ridge approximation method,
which tries to find an approximation for a symmetric positive
semidefinite matrix. We have also developed probabilistic
formulations for this method. The probabilistic formulation not
only provides a statistical interpretation but also leads us to an
efficient EM iterative procedure for the matrix ridge approximation.
The matrix ridge approximation with the EM iteration has potentially broad
applicability in machine learning problems that involve the
inversion or spectral decomposition of a large-scale positive semidefinite matrix.
In particular, we have empirically illustrated the effectiveness
and efficiency of the matrix ridge approximation in the case of spectral clustering and Gaussian process regression.

The support vector machine (SVM) and Gaussian process classification (GPC) are  two classical kernel classification methods.  When applying them to large-scale data sets,
we also meet a computational challenge.
The matrix ridge approximation technique is a potentially useful approach for handling this challenge.
We will study this issue in future work. Recall that each EM iteration for the matrix ridge approximation
takes time $O(m^2 q)$ and it mainly involves matrix multiplications.
To make the method more efficient, we can consider the   parallel implementation of the matrix multiplications.

%

\appendix

\section{Several Lemmas}

In order to prove the theorems, we first present several lemmas that will be used.

\begin{lemma} \label{lem:4}
Suppose $\B \in \BR^{m{\times}m}$. Let $c_i+\imath d_i$ for $i=1, \ldots, m$ be the eigenvalues of $\B$ where $\imath^2=-1$ and the $c_i, d_i \in \BR$. Then,
\begin{enumerate}
\item[\emph{(i)}] \; $\tr(\B)=\sum_{i=1}^m c_i$ and $\tr(\B^2)=\sum_{i=1}^m (c^2_i- d_i^2)$ .
\item[\emph{(ii)}] \; $\sum_{i=1}^m c_i^2 + d_i^2 \leq \tr(\B\B^T)$, $\sum_{i=1}^m c_i^2 \leq \frac{1}{2}\tr(\B \B {+} \B\B^T)$, and
$\sum_{i=1}^m d_i^2 \leq \frac{1}{2}\tr(\B\B^T{-}\B\B)$.
\end{enumerate}
\end{lemma}
\begin{proof} It is obvious that $c_i+\imath d_i$ is the eigenvalue of $\B$ iff $c_i-\imath d_i$ is the eigenvalue of $\B$.
Accordingly, we have Part~(i).

In addition, let the Schur factorization of $\B$ be $\B=\Q \T \Q^{*}$ where $\Q$ is unitary and $\T$ is upper-triangular
with the eigenvalues of $\B$ at the diagonals. Thus,
\[
\tr(\B\B^T) = \tr(\T \T^{*}) \geq \sum_{i=1}^m (c_i+\imath d_i) (c_i-\imath d_i) = \sum_{i=1}^m (c_i^2 + d_i^2).
\]
In addition, we also have
\[
\frac{1}{2}\tr(\B\B^T+ \B \B) =\frac{1}{4} \tr((\T +\T^{*})(\T +\T^{*})) \geq \sum_{i=1}^m c_i^2
\]
and
\[
\frac{1}{2}\tr(\B\B^T- \B \B) =\frac{1}{4} \tr((\T - \T^{*})(\T - \T^{*})) \geq \sum_{i=1}^m d_i^2.
\]
The proof completes.
\end{proof}

We now turn to our proposed approach and follow the notations in Table~\ref{tab:nota}. Without loss of generality,
we only consider the case that $\b\neq 0$ and $\b^T\b=1$. In this
case, $\BP=\I_m- \b\b^T$ is idempotent, symmetric and of rank
$m{-}1$. Thus we can express it as $\BP = \Ps \left[
\begin{array}{cc} \I_{m{-}1} & \0 \\ \0 & \0
\end{array} \right] \Ps^T$
where $\Ps^T \Ps = \Ps \Ps^T =\I_m$. Let $\Ps_1$ be an
$m{\times}(m{-}1)$ matrix containing the first $m{-}1$ columns
of~$\Ps$. Then $\Ps=[\Ps_1, \b]$ so that $\Ps_1^T \Ps_1 =\I_{m{-}1}$,
$\Ps_1^T \b =\0$ and $\BP = \Ps_1 \Ps_1^T$.

In order to prove the theorems given in Section~\ref{sec:golub}, we
use the same notation as in Section~\ref{sec:golub}. Moreover, we
here and later denote $\Z= \Ps_1^T \A$ ($(m{-}1){\times}q)$), $\G =
\Ps_1^T\T \Ps_1 = \Ps_1^T \BP \M \BP \Ps_1=\Ps_1^T \M \Ps$
($(m{-}1){\times}(m{-}1)$) and $\Tha= \Ps_1^T \Oma \Ps_1 = \Ps_1^T (\A
\A^T+ \delta \I_m) \Ps_1 = \Z \Z^T+ \delta \I_{m{-}1}$
($(m{-}1){\times}(m{-}1)$). With these notations, we present the
following several lemmas.

\begin{lemma} \label{lem:3} Let $\lambda(\C)$ be the set of the all eigenvalues
of $\C$. Then $\lambda(\T)= \lambda(\G) \cup \{0\}$. Furthermore, if
$\ph$ is the eigenvector of $\G$ associated with its eigenvalue
$\gamma$, then $\Ps_1\ph$ is the eigenvector of $\T$ associated with
its eigenvalue $\gamma$. Conversely, if $\u$ satisfying $\u^T\b=0$ is
the eigenvector of $\T$ associated with its eigenvalue $\gamma$,
then $\Ps_1^T \u$ is the eigenvector
of $\G$ associated with its eigenvalue $\gamma$.
\end{lemma}

\begin{proof}
Recall that
\[
\T = \BP \T \BP = \Ps \left[
\begin{array}{cc} \Ps_1^T \T \Ps_1 & \0 \\ \0 & 0
\end{array} \right] \Ps^T = \Ps \left[
\begin{array}{cc} \G & \0 \\ \0 & 0
\end{array} \right] \Ps^T.
\]
Thus, $\lambda(\T)= \lambda(\G) \cup \{0\}$. Letting $\G \ph =
\gamma \ph$, we have
\begin{eqnarray*}
\BS \Ps_1 \ph &=& \Ps \left[
\begin{array}{cc} \G  & \0 \\ \0 & 0
\end{array} \right] \Ps^T \Ps_1  \ph = \Ps \left[
\begin{array}{cc} \G  & \0 \\ \0 & 0
\end{array} \right] \left[ \begin{array}{c} \I_{m{-}1} \\ 0
\end{array} \right]  \ph  \\
& =& \Ps \left[ \begin{array}{c} \G    \\ 0
\end{array} \right]  \ph = \Ps \left[ \begin{array}{c} \I_{m{-}1} \\ 0
\end{array} \right]  \G  \ph = \Ps_1 \G \ph = \gamma  \Ps_1 \ph,
\end{eqnarray*}
which shows that $\Ps_1 \ph$ is the eigenvectors of $\T$. Also,
since
\[
\G \Ps_1^T \u = \Ps_1^T\T \Ps_1 \Ps_1^T \u = \Ps_1^T\T \u = \gamma
\Ps_1^T \u
\]
$\Ps_1^T \u$ is the eigenvector of $\G$ associated with its
eigenvalue $\gamma$.
\end{proof}

\begin{lemma} \label{lem:2} Assume that $k$ is an arbitrary integer. Then,
\begin{enumerate}
\item[\emph{(i)}] $\Ps_1^T \Oma^{-k} \Ps_1 = \Tha^{-k}$, $\b^T \Oma^{-k} \Ps_1 = \0$ and $\b^T \Oma^{-k} \b = \delta^{-k}$;
\item[\emph{(ii)}] $\tr(\Oma^{-k}) = \tr(\Tha^{-k}) + \delta^{-k}$.
\end{enumerate}
\end{lemma}

\begin{proof}
As for (i), we first have
\begin{align*}
\Ps_1^T \Oma^{-1} \Ps_1 &= \Ps_1^T (\A \A^T + \delta \I_m)^{-1} \Ps_1 =
\Ps_1^T(\delta^{-1}\I_m - \delta^{-1} \A(\delta\I_q + \A^T \A)^{-1}
\A^T) \Ps_1 \\
& = \delta^{-1}\I_{m{-}1} - \delta^{-1} \Z (\delta\I_q + \Z^T
\Z)^{-1} \Z = (\delta \I_{m{-}1} + \Z \Z)^{-1} = \Tha^{-1}
\end{align*}
due to $\A^T\A= \A^T\BP \A= \A^T\Ps_1 \Ps_1^T \A = \Z^T\Z$. Assume that
$\Ps_1^T \Oma^{1-l} \Ps_1 = \Tha^{1-l}$ for some positive integer
$l$. Then
\[
\Ps_1^T \Oma^{-l} \Ps_1 = \Ps_1^T \Oma^{1-l} (\Ps_1 \Ps_1^T + \b \b^T)
\Oma^{-1} \Ps_1 = \Tha^{-l} + \Ps_1^T \Oma^{1-l} \b \b^T \Oma^{-1}
\Ps_1 = \Tha^{-l}
\]
due to $\b^T \Oma^{-1} \Ps_1=\0$. Thus, we obtain $\Ps^T \Oma^{-k}
\Ps_1 = \Tha^{-k}$ by the induction. Similarly, we $\b^T \Oma^{-k}
\Ps_1 = \0$ and $\b^T \Oma^{-k} \b = \delta^{-k}$.

Finally, it follows from (i) that
\[
\tr(\Oma^{-k}) = \tr\left(\left[\begin{array}{c} \Ps_1^T \\ \b^T
\end{array} \right] \Oma^{-k1} [\Ps_1, \b] \right) = \tr(\Tha^{-k})
+ \delta^{-k}.
\]
\end{proof}

\section{Proof for Theorem~\ref{thm:lse2}} \label{ap:a}

In order to prove Theorem~\ref{thm:lse2},  we present a more general alternative  which is based on
two  variants of $F$ and $G$. In particular, the first variant is
\[
F_1(\A, \delta) = \|\M {-} \A \A^T {-} \delta \I_m\|_F^2 =
\tr\big((\M {-} \A \A^T {-} \delta \I_m)^2\big),
\]
while the second variant is
\[
G_1(\A, \delta) = \log |\A \A^T + \delta \I_m| + \tr(( \A \A^T +
\delta \I_m)^{-1} \M).
\]
Obviously,  $F_1$ and $F$ (or $G_1$ and $G$) become identical
when $\b=\0$.
The minimizers of $F_1$ as well as $G_1$ are given in the following
theorem.

\begin{theorem} \label{thm:lse}
Let $\gamma_1\geq \cdots \geq \gamma_q \geq \cdots \geq \gamma_{m}$
\emph{($\geq 0$)} be the eigenvalues of $\T = \BP \M \BP$, $\V$ be
an arbitrary $q{\times}q$ orthogonal matrix, ${\Gam_q}$ be a
$q{\times}q$ diagonal matrix containing the first $q$ principal
(largest) eigenvalues $\gamma_i$, and ${{\U_q}}$ be an $n{\times}q$
column-orthonormal matrix in which the $q$ column vectors are the principal
eigenvectors corresponding to ${\Gam_q}$. Assume that $\delta>0$ and
that $\A \in \BR^{m{\times}q}$ \emph{($q< \min(m, p)$)} is of full
column rank and satisfies $\A^T\b=\0$. If the following conditions are satisfied
\begin{equation} \label{eq:conds}
\gamma_i > \frac{1}{m{-}q} \Big(\b^T\M\b+\sum_{j=q{+}1}^{m} \gamma_j
\Big), \quad \mbox{ for } i=1, \ldots, q,
\end{equation}
then the strict local minimum of $F_1(\A, \delta)$ and $G_1(\A, \delta)$ w.r.t.\ $(\A, \delta)$ are respectively obtained  when
\[
\widehat{\A} = {\U_q} ({\Gam_q} - \hat{\delta} \I_q)^{1/2} \V \quad
\mbox{and} \quad  \hat{\delta} = \frac{1}{m{-}q} \Big [\b^T \M \b+
\sum_{j=q+1}^{m}\gamma_j \Big].
\]
\end{theorem}

Note that $\BP\BS\BP=\BS$ and $\b^T\BS\b=0$. Thus, when viewing $\BS$ as $\M$ in Theorem~\ref{thm:lse},
we immediately obtain Theorem~\ref{thm:lse2} from
Theorem~\ref{thm:lse}.

Theorem~\ref{thm:lse}  shows the connection between the
estimates of $\A$ and $\delta$ based on the minimizations of $F_1$
and $G_1$. In particular, the estimates of $\A$ and $\delta$ via
minimizing $F_1$ are equivalent to those of $\A$ and $\delta$ via
minimizing $G_1$.
We note that the minimizer $(\widehat{\A}, \hat{\delta})$ of
$G_1$ under $\b=\0$ was given in
\citet{Magnus:1999}. The conditions in (\ref{eq:conds}) aim to
ensure that $({\Gam_q} - \hat{\delta} \I_q)^{1/2}$ exists and
$\widehat{\A}$ is of full column rank. In the case that $\b=\0$,
$\gamma_q> \gamma_{q{+}1}$ suffices for  the conditions. In fact,
they are always satisfied whenever there is at least one $\gamma_j$ where
$j \in \{q{+}1, \ldots, m\}$ such that $\gamma_q > \gamma_j>0$.
Thus, the conditions in (\ref{eq:conds}) are trivial when $\b=\0$.

However, the conditions are not always satisfied when $\b^T\b=1$. For example,
let
\[
\M=[\b, \Ps_1]\left[\begin{array}{cc} 1+\alpha^2 & \0 \\ \0 &
\alpha^2 \I_{m{-}1}
\end{array} \right] \left[\begin{array}{c} \b^T \\ \Ps_1^T
\end{array} \right]
\]
for $\alpha \neq 0$ such that $\Ps_1^T\b=\0$ and
$\Ps_1^T\Ps_1=\I_{m{-}1}$. It is clear that $\b^T\M \b = 1+\alpha^2$
and $\T=\BP \M \BP= \alpha^2 \BP$. This implies that the eigenvalues
$\gamma_i$ of $\T$ are $\alpha^2$ with multiplicity $m{-}1$ and $0$ with multiplicity 1.
As a result, for any $i\leq q<m$, we always have
\[
\gamma_i=\alpha^2< \alpha^2 + 1/(m-q)= \frac{1}{m{-}q}
\Big(\b^T\M\b+\sum_{j=q{+}1}^{m} \gamma_j \Big).
\]
Thus, the condition in (\ref{eq:conds}) is not satisfied. Consequently, this condition would
limit the use of $F_1$ and $G_1$ in the matrix ridge approximation.
This is the reason why we employ $F$ and $G$ instead of $F_1$ and $G_1$ respectively.

\subsection{Proof for the Minimizer of $F_1(\A, \delta)$ w.r.t.\ $(\A, \delta)$}

Consider the Lagrangian function of
\[
L= \tr(\M - \A \A^T-\delta \I_m)^2 + 4 \b^T \A \a
\]
where $\a$ is a $q{\times}1$ vector of Lagrangian multipliers. We now compute
\begin{align*}
d L 
    & = - 2 \tr\big[(\M - \A \A^T-\delta \I_m) ((d \A) \A^T +  \A (d \A^T))\big] + 4 \b^T (d \A) \a \\
    & = - 4 \tr(\A^T (\M - \A \A^T-\delta \I_m) (d \A)) + 4 \b^T (d \A) \a,  \\
d L & = -2 \tr(\M - \A \A^T-\delta \I_m)  d \delta.
\end{align*}
Using the first-order condition, we obtain
\begin{align*}
  - \A^T (\M - \A \A^T-\delta \I_m) + \a \b^T  &= \0,  \\
 \tr(\M - \A \A^T-\delta \I_m)  &= 0.
\end{align*}
Postmultiplying the above first equation by $\b$, we obtain $\a =
\A^T \M \b$ because of $\A^T \b=\0$. As a result, we have
\[
\T \A =  \A (\A^T \A + \delta \I_q),
\]
where $\T=\BP \M \BP$. Assume  the spectral decomposition of $\A^T\A$
as $\A^T\A = \V \Lam \V^T$. Hence,
\[
\T \A \V \Lam^{-\frac{1}{2}} =  \A \V  \Lam^{-\frac{1}{2}} (\Lam +
\delta \I_q).
\]
This implies that the diagonal elements of $\Lam + \delta \I_q$ are
the $q$ eigenvalues of $\T$, and  $\A \V \Lam^{-\frac{1}{2}}$ is a
corresponding matrix of orthonormal eigenvectors. This motivates us
to define $\Gam_q= \Lam + \delta \I_q$ and $\U_q= \A \V
\Lam^{-\frac{1}{2}}$. That is, $\Lam = \Gam_q - \delta \I_q$ and
$\widehat{\A}= \U_q \Lam^{\frac{1}{2}} \V^T$.

On the other hand, we have
\[
\tr(\M) = \tr(\Ps \M \Ps) = \tr(\Ps_1^T \M \Ps_1) + \b^T \M \b
\] and
$ \tr(\T) = \tr(\BP \M \BP) = \tr(\M \BP) = \tr(\M \Ps_1 \Ps_1^T) =
\tr(\Ps_1^T \M \Ps_1)$. It then follows from $\tr(\M - \A \A^T-\delta
\I_m) = 0$ that
\[
m \delta + \sum_{i=1}^q \gamma_i - q \delta = \sum_{i=1}^m \gamma_i
+  \b^T\M \b.
\]
Thus we let $\hat{\delta} = \frac{1}{m{-}q} \big(\sum_{i=q{+}1}^m
\gamma_i +  \b^T \M \b \big)$. Condition~\ref{eq:conds} shows that
${\Lam}^{\frac{1}{2}} = (\Gam_q - \hat{\delta}
\I_q)^{\frac{1}{2}}$ exists and $\widehat{\A}= \U_q
{\Lam}^{\frac{1}{2}} \V^T$ is of full column rank.

To verify that $(\widehat{\A}, \hat{\delta})$ is a minimizer of
$F_1(\A, \delta)$, we compute the Hessian matrix of $L$ w.r.t.\
to $(\A, \W)$. Let $\vecd(\A)=(y_{11}, \ldots, y_{m1}, y_{12},
\ldots, y_{m q})^T$. The Hessian matrix is then given by
\begin{align*}
H  (\A, \delta) & \triangleq \left[\begin{array}{cc} \frac{\partial^2 L}
{{\partial \vecd(\A)} {\partial \vecd(\A)}^T} &  \frac{\partial^2
L} {{\partial \vecd(\A)} {\partial \delta}} \\
\frac{\partial^2 L} {{\partial \delta} {\partial \vecd(\A)}^T} &
\frac{\partial^2 L} {{\partial \delta^2} }
\end{array} \right] \\
& = 4 \left[\begin{array}{cc} \I_q {\otimes} (\delta \I_m + \A \A^T
{-}\M)  +  \A^T\A {\otimes} \I_m + (\A^T{\otimes}\A) \C_{mq} & \vecd(\A) \\
\vecd(\A)^T & \frac{m}{2}
\end{array} \right],
\end{align*}
where $\C_{mq}$ is the $mq{\times} mq$ commutation such that
$\C_{mq} \vecd(\B) = \vecd(\B^T)$ for any $m{\times}q$ matrix $\B$.

Let $\X$ be an arbitrary nonzero $m{\times}q$ matrix such that $\X^T
\b= \0$, and $a$ be a nonzero real number. Hence,
\begin{align*}
\zeta & \triangleq \frac{1}{4}[\vecd(\X)^T, a] H(\widehat{\A}, \hat{\delta})
[\vecd(\X)^T, a]^T \\
& = \tr(\X^T(\hat{\delta}\I_m {+} \widehat{\A} \widehat{\A}^T{-}\M)
\X) + \tr(\X^T \X \widehat{\A}^T\widehat{\A})+ \tr(\X \widehat{\A}^T\X
\widehat{\A}^T) + 2 a \tr(\X \widehat{\A}^T) + \frac{m}{2} a^2 \\
& = \tr(\X^T(\hat{\delta}\I_m {+} \widehat{\A} \widehat{\A}^T{-}\T)
\X) +\tr(\X^T \X \widehat{\A}^T\widehat{\A})+  \tr(\X \widehat{\A}^T\X
\widehat{\A}^T) + 2 a \tr(\X \widehat{\A}^T) + \frac{m}{2} a^2
\end{align*}
due to $\X^T\M\X=\X^T\T\X$.

Let $\T=\U \Gam \U^T$ where $\U=[\U_q, \U_2]$ and $\Gam=\diag(\Gam_q,
\Gam_2)$ such that $\U_2^T \U_q=\0$, $\U_2^T\U_2=\I_{m{-}q}$ and
$\Gam_2=\diag(\gamma_{q{+}1}, \ldots, \gamma_m)$. Thus,
\[
\hat{\delta}\I_m {+} \widehat{\A} \widehat{\A}^T{-}\T = [\U_q, \U_2]
\left[\begin{array}{cc} \0 & \0 \\
\0 & \hat{\delta} \I_{m{-}q} {-} \Gam_2  \end{array} \right]
 \left[\begin{array}{c} \U_q^T \\ \U_2^T \end{array}\right] = \U_2( \hat{\delta} \I_{m{-}q} {-}
 \Gam_2)\U_2^T.
\]
Furthermore, we have $\tr(\X \hat{\A}^T)=\tr(\B_q \Lam^{\frac{1}{2}})$,
\[
\tr(\X^T(\hat{\delta}\I_m {+} \widehat{\A} \widehat{\A}^T{-}\T)\X) =
\tr(\B_2^T (\hat{\delta} \I_{m{-}q} {-} \Gam_2) \B_2),
\]
$\tr(\X \widehat{\A}^T\X \widehat{\A}^T)= \tr(\U_q^T \X \V
\Lam^{\frac{1}{2}} \U_q^T \X \V \Lam^{\frac{1}{2}}) = \tr(\B_q
\Lam^{\frac{1}{2}} \B_q \Lam^{\frac{1}{2}})$ and
\[
\tr(\X^T\X\widehat{\A}^T\widehat{\A})  = \tr(\V^T \X^T \U \U^T \X \V
\Lam) =  \tr(\B^T \B \Lam)  = \tr(\B_q^T \B_q \Lam) + \tr(\B_2^T \B_2
\Lam) \] where $\B_q=\U_q^T\X \V$ ($q{\times}q$),  $\B_2=\U_2^T\X \V$
($(m{-}q){\times}q$), and $\B=\U^T\X \V =[\B_q^T, \B_2^T]^T =[\b_{1},
\ldots, \b_{m}]^T$ ($m{\times}q$). Accordingly, we obtain
\begin{align*}
\zeta
& = \tr(\B_2^T (\hat{\delta} \I_{m{-}q} {-} \Gam_2) \B_2) + \tr(\B_2^T
\B_2 \Lam) + \\
& \quad  \tr(\B_q^T \B_q \Lam) + \tr(\B_q \Lam^{\frac{1}{2}} \B_q
\Lam^{\frac{1}{2}}) + 2 a \tr(\B_q \Lam^{\frac{1}{2}}) + \frac{m}{2} a^2.
\end{align*}
Recall that
$\Lam=\Gam_q{-} \hat{\delta} \I_q$. It is easily verified that  $\tr(\B_2^T (\hat{\delta} \I_{m{-}q}
{-} \Gam_2) \B_2) + \tr(\B_2^T \B_2 \Lam) \geq 0$. In addition, let  the real parts of the eigenvalues
of $\B_q \Lam^{\frac{1}{2}}$ be $\eta_i$ for $i=1, \ldots, q$. It follows from Lemma~\ref{lem:4} that
\begin{eqnarray*}
\lefteqn{\tr(\B_q^T \B_q \Lam) + \tr(\B_q \Lam^{\frac{1}{2}} \B_q
\Lam^{\frac{1}{2}}) + 2 a \tr(\B_q \Lam^{\frac{1}{2}}) + \frac{m}{2} a^2 } \\
&\geq & \sum_{i=1}^q  [ 2 \eta_i^2 + 2a \eta_i] + \frac{m}{2} a^2  = \sum_{i=1}^q  \frac{1}{2}(2\eta_i + a)^2  + \frac{m-q}{2} a^2 \\
&> & 0.
\end{eqnarray*}

In summary, we obtain $[\vecd(\X^T)^T, a] H(\widehat{\A},
\hat{\delta}) [\vecd(\X^T)^T, a]^T > 0$. Thus, this implies that
$(\widehat{\A}, \hat{\delta})$ is the strict local minimizer of $F_1(\A,
\delta)$.

Replacing $\BS$ for $\M$ in $F_1(\A, \delta)$ and considering
$\BP\BS \BP = \BS$, we immediately obtain the strict local minimizer of $F(\A,
\delta)$. In this case,  we have $\hat{\delta}=\frac{1}{m{-}q}
\sum_{i=q+1}^{m} \gamma_i$ due to $\BS^T\b=\0$,.

\subsection{Proof for the Minimizer of $G_1(\A, \delta)$ w.r.t.\ $(\A, \delta)$}

To prove that the $(\widehat{\A}, \hat{\delta})$ is also the
minimizer of $G_1(\A, \delta)$, we consider  the following the
Lagrangian function:
\[
L(\A, \delta) = \log|\Oma| + \tr(\Oma^{-1} \M) + 2 \b^T \A \a
\]
where $\a$ is the $q{\times}1$ vector of Lagrangian multipliers. We
have
\begin{align*}
d L &= \tr(\Oma^{-1} (d\Oma)) - \tr \big(\M \Oma^{-1}
(\d \Oma) \Oma^{-1} \big) + 2 \b^T (d\A) \a \\
&=  \tr\big( \Oma^{-1} ((d \A) \A^T {+} \A (d \A^T)) \big) - \tr \big(
\Oma^{-1} \M \Oma^{-1} ((d \A) \A^T {+} \A (d
\A^T))  \big)  + 2 \b^T (d\A) \a \\
&= 2 \tr\big( \A^T \Oma^{-1} (d \A) \big) -  2 \tr \big( \A^T
\Oma^{-1} \M \Oma^{-1} (d \A) \big) + 2 \b^T (d\A) \a , \\
d L &= \tr\big( \Oma^{-1} (d \delta)  \big) - \tr \big( \Oma^{-1} \M
\Oma^{-1} (d \delta) \big).
\end{align*}
Then, using the first-order condition, we have $\tr(\Oma^{-1})-
\tr(\Oma^{-1} \M \Oma^{-1})=0$ and
\[
\A^T \Oma^{-1} -   \A^T \Oma^{-1} \M \Oma^{-1} + \a \b^T =\0.
\]
Postmultiplying the above equation by $\b$, we obtain $\a = \left(
\A^T \Oma^{-1} \M \Oma^{-1} {-} \A^T \Oma^{-1} \right) \b$. As a
result, we have the first-order condition for $\A$ as
\[
\A^T \Oma^{-1} \BP = \A^T \Oma^{-1} \M \Oma^{-1} \BP,
\] which is
equivalent to that
\begin{align*}
\A^T \Ps_1 \Ps_1^T \Oma^{-1} \Ps_1 \Ps_1^T &=  \A^T \Ps_1 \Ps_1^T
\Oma^{-1} (\Ps_1 \Ps_1^T+ \b \b^T) \M  (\Ps_1 \Ps_1^T + \b \b^T) \Oma^{-1} \Ps_1 \Ps_1^T \\
&= \A^T \Ps_1 \Ps_1^T \Oma^{-1} \Ps_1 \Ps_1^T \T  \Ps_1 \Ps_1^T
\Oma^{-1} \Ps_1 \Ps_1^T
\end{align*}
due to $\BP=\Ps_1 \Ps_1^T$, $\T=\BP\M\BP$, $\A^T\BP = \A^T$ and $\b^T
\Oma^{-1} \Ps_1=\0$. We thus obtain
\[
\Ps_1^T \T  \Ps_1 \Ps_1^T \Oma^{-1} \Ps_1  \Z = \Z,
\]
where $\Z =\Ps_1^T\A$. According to Lemma~\ref{lem:2}, the
first-order condition for $\A$ becomes
\begin{equation} \label{eq:mcl_f1}
\G  \Tha^{-1} \Z = \Z.
\end{equation}
where $\G = \Ps_1^T \T \Ps = \Ps_1^T \BP \M \BP \Ps = \Ps_1^T \M \Ps$.
In addition,  from Lemma~\ref{lem:2}, we have
\begin{align*}
\tr(\Oma^{-1} \M \Oma^{-1} ) & = \tr\left(\left[\begin{array} {c}
\Ps_1^T \\ \b^T
\end{array} \right] \Oma^{-1}[\Ps_1, \b] \left[\begin{array} {c}
\Ps_1^T \\ \b^T
\end{array} \right] \M [\Ps_1, \b]  \left[\begin{array} {c}
\Ps_1^T \\ \b^T
\end{array} \right] \Oma^{-1} [\Ps_1, \b] \right) \\
& = \tr\left(\left[\begin{array} {c c} \Tha^{-1} & \0 \\
\0 & \delta^{-1}
\end{array} \right] \left[\begin{array} {cc}
\Ps_1^T \M \Ps_1 & \Ps_1^T \M \b \\ \b^T \M \Ps_1 & \b^T \M \b
\end{array} \right] \left[\begin{array} {c c} \Tha^{-1} & \0 \\
\0 & \delta^{-1}
\end{array} \right] \right) \\
& = \tr(\Tha^{-1} \Ps_1^T \M \Ps \Tha^{-1}) + \delta^{-2} \b^T \M \b.
\end{align*}
The first-order condition for $\delta$ thus becomes
\begin{equation} \label{eq:mcl_f2}
\delta^2 \big[\tr(\Tha^{-1}) - \tr(\Tha^{-1} \G \Tha^{-1}) \big]+
\delta - \b^T \M \b =0.
\end{equation}

It follows from $\Z \Z^T= \Tha- \delta \I_{m{-}1}$ that
\[
\Z \Z^T \Z = \G  \Tha^{-1} \Z \Z^T\Z = \G \Z - \delta \Z,
\]
which yields
\begin{equation} \label{eq:aby0}
\G \Z = \Z(\delta \I_q+\Z^T\Z).
\end{equation}
Assume that the rank of $\Z$ is $q$ ($\leq m{-}1$). There exists a
semi-orthogonal $q{\times}q$ matrix $\V$ ($\V\V^T=\I_q$) and a
$q{\times}q$ diagonal matrix $\Lam={\diag}(\lambda_1, \ldots,
\lambda_q)$ such that
\[
\Z^T  \Z = \V \Lam \V^T.
\]
It is clear that $\V$ and $\Lam$ are the eigenvector matrix and
eigenvalue matrix of $\Z^T  \Z$, respectively. Then we can rewrite
(\ref{eq:aby0}) as
\[
\G  \Z \V = \Z \V (\delta\I_q + \Lam)
\]
which gives
\[
 \G \Z \V \Lam^{-1/2} =  \Z \V
\Lam^{-1/2} (\delta \I_q + \Lam).
\]
Denote $\Ph_q = \Z \V \Lam^{-1/2}$ ($(m{-}1){\times}q$). It is easy
to see $\Ph_q^T\Ph_q = \I_q$. Thus,  $\Ph_q$ and $\delta\I_q{+}\Lam$
are the eigenvector and eigenvalue matrices of $\G$, respectively.
This motivates us to equalize $\delta\I_q{+}\Lam = \Gam_q$ and $\Z
\V \Lam^{-1/2} = \Ps_1^T \U_q$. That is, we let $\widehat{\A} = \U_q
(\Gam_q- \delta \I_q)^{1/2} \V^T$.

On the other hand, since
\[
\Tha^{-1} = \delta^{-1} \I_{m{-}1} - \delta^{-1} \Z(\delta \I_q +
\Z^T \Z)^{-1} \Z^T
\]
and from (\ref{eq:aby0}), we have
\[
\G \Tha^{-1}  = \delta^{-1} \G  -  \delta^{-1}\G \Z(\delta \I_q +
\Z^T \Z)^{-1} \Z^T  = \delta^{-1} (\G  - \Z \Z^T).
\]
Hence
\[
\delta^{2} (\Tha^{-1} \G \Tha^{-1} - \Tha^{-1}) = \G - \Z \Z^T -
\delta \I_{m{-}1}.
\]
Combining this equation with (\ref{eq:mcl_f2}) yields
\[
m \delta = \tr(\G) - \tr(\Z^T\Z) + \b^T\M\b.
\]
We thus set $\hat{\delta}=\frac{1}{m-q}(\b^T\M\b+
\sum_{j=q{+}1}^{m{-}1} \gamma_j)$.

It is clearly seen that $(\hat{\delta}, \widehat{\A})$ satisfy the
first-order conditions of $L$ w.r.t.\ $(\delta, \A)$. To verify that
$(\hat{\delta}, \widehat{\A})$ are the minimizer of $g(\A, \delta)$,
we compute
\begin{align*}
\frac{1}{2} d^2 L = & \tr[(d \A^T) \Oma^{-1} (d \A)] - \tr[\A^T
\Oma^{-1} (d \A) \A^T \Oma^{-1} (d \A)] - \tr[\A^T \Oma^{-1} \A (d
\A^T) \Oma^{-1}(d
\A)] \\
&  - \tr[(d \A^T) \Oma^{-1} \M \Oma^{-1} (d \A)] + \tr[\A^T \Oma^{-1}
(d \A) \A^T \Oma^{-1} \M \Oma^{-1} (d \A)] \\
& + \tr[\A^T \Oma^{-1} \A (d \A^T) \Oma^{-1} \M \Oma^{-1} (d \A)] +
\tr[\A^T \Oma^{-1} \M \Oma^{-1} (d \A) \A^T \Oma^{-1}  (d \A)] \\
& + \tr[\A^T \Oma^{-1} \M \Oma^{-1} \A (d \A^T) \Oma^{-1}  (d \A)] \\
\frac{1}{2} d^2 L = & - \tr[\A^T \Oma^{-2} (d \A) ] (d \delta) +
\tr[\A^T \Oma^{-1} \M \Oma^{-2} (d \A)] (d \delta) + \tr[\A^T
\Oma^{-2} \M \Oma^{-1} (d \A)] (d \delta), \\
\frac{1}{2} d^2 L = & - \frac{1}{2} \tr[\Oma^{-2}] (d \delta) (d
\delta) + \tr[\Oma^{-3} \M] (d \delta) (d \delta).
\end{align*}
We thus have the Hessian matrix:
\[
 H  (\A, \delta)  \triangleq \left[\begin{array}{cc}
\frac{\partial^2 L} {{\partial \vecd(\A)} {\partial \vecd(\A)}^T} &
\frac{\partial^2
L} {{\partial \vecd(\A)} {\partial \delta}} \\
\frac{\partial^2 L} {{\partial \delta} {\partial \vecd(\A)}^T} &
\frac{\partial^2 L} {{\partial \delta^2} }
\end{array} \right]
\]
where $\frac{1}{2}\frac{\partial^2 L} {{\partial \delta^2} } =
\tr[\Oma^{-3} \M] - \frac{1}{2} \tr[\Oma^{-2}]$, \[\frac{1}{2}
\frac{\partial^2 L} {{\partial \vecd(\A)} {\partial \delta}} = [\I_q
{\otimes} (\Oma^{-1} \M \Oma^{-2} {+} \Oma^{-2} \M \Oma^{-1} {-}
\Oma^{-2} )] \vecd(\A), \]
\begin{align*}
\frac{1}{2} \frac{\partial^2 L} {{\partial \vecd(\A)} {\partial
\vecd(\A)}^T}  = & [\I_q {-} \A^T \Oma^{-1} \A] {\otimes} [\Oma^{-1}
{-} \Oma^{-1} \M \Oma^{-1}] +  \A^T \Oma^{-1} \M \Oma^{-1} \A
{\otimes} \Oma^{-1}
\\ & {+} \C_{q m }\big[\Oma^{{-}1} \A {\otimes}
\A^T \Oma^{-1} \M \Oma^{{-}1} {+} \Oma^{-1} \M \Oma^{-1} \A
{\otimes}\A^T \Oma^{{-}1} {-} \Oma^{-1} \A {\otimes} \A^T
\Oma^{-1}\big].
\end{align*}

Given an arbitrary nonzero matrix $\X \in \BR^{m{\times}q}$ such
that $\X^T \b= \0$, and  a nonzero number $a \in \BR$, we have
\begin{align*}
B & \triangleq \frac{1}{2}[\vecd(\X)^T, a] H({\A}, {\delta})
[\vecd(\X)^T, a]^T \\
& = \tr\big[\X (\I_q {-} \A^T \Oma^{-1} \A) \X^T (\Oma^{-1} {-}
\Oma^{-1} \M \Oma^{-1}) \big] +  \tr\big[\X \A^T \Oma^{-1} \M \Oma^{-1} \A \X^T \Oma^{-1} \big] \\
& \quad + 2 \tr(\X \A^T \Oma^{-1}  \X \A^T
\Oma^{-1} \M \Oma^{-1}) - \tr\big[\X \A^T \Oma^{-1} \X \A^T \Oma^{-1}\big] \\
& \quad + 2 a \tr\big[{\X}^T (\Oma^{-1} \M \Oma^{-2} {+}
\Oma^{-2} \M \Oma^{-1} {-} \Oma^{-2} ) \A  \big]  + \big[\tr(\Oma^{-3} \M) - \frac{1}{2} \tr(\Oma^{-2})\big] a^2 \\
& = \tr\big[\X_1 (\I_q {-} \Z^T \Tha^{-1} \Z) \X_1^T (\Tha^{-1} {-}
\Tha^{-1} \G \Tha^{-1}) \big] +  \tr\big[\X_1 \Z^T \Tha^{-1} \G \Tha^{-1} \Z \X_1^T \Tha^{-1} \big] \\
& \quad + 2 \tr(\X_1 \Z^T \Tha^{-1}  \X_1 \Z^T
\Tha^{-1} \G \Tha^{-1}) - \tr\big[\X_1 \Z^T \Oma^{-1} \X_1 \Z^T \Oma^{-1}\big] \\
& \quad  + 2 a \tr\big[\X_1^T (\Tha^{-1} \G \Tha^{-2} {+} \Tha^{-2}
\G \Tha^{-1} {-} \Tha^{-2} ) \Z  \big]  +
\frac{1}{2}\big[2\tr(\Oma^{-3} \M) - \tr(\Tha^{-2})  - \delta^{-2}
\big] a^2 \\
& = \tr\big[\X_0 (\I_q {-} \Z^T \Tha^{-1} \Z) \X_0^T (\Tha {-}
\G) \big] +  \tr\big[\X_0 \Z^T \Tha^{-1} \G \Tha^{-1} \Z \X_0^T \Tha \big] \\
& \quad + 2 \tr(\X_0 \Z^T \X_0 \Z^T \Tha^{-1} \G) - \tr\big[\X_0 \Z^T
\X_0 \Z^T\big]  + 2 a \tr\big[\X_0^T (\G \Tha^{-2} {+} \Tha^{-1} \G
\Tha^{-1} {-} \Tha^{-1} ) \Z  \big] \\
& \quad   + \frac{1}{2}\big[2\tr(\Tha^{-3} \G) - \tr(\Tha^{-2})  -
\delta^{-2} \big] a^2 + \frac{a^2}{\delta^3} \b^T \M \b
\end{align*}
where $\X_1= \Ps_1^T\X$ and $\X_0=\Tha^{-1} \X_1$. Here we use the
fact that $\X = \BP \X = \Ps_1 \Ps_1^T\X$, $\b^T \Oma^{-1} \Ps_1=\0$,
$\G=\Ps_1^T\M \Ps_1$ and  $\tr(\Oma^{-3} \M) = \tr(\Tha^{-3} \G) +
\delta^{-3} \b^T\M\b$.

Recall that the eigenvalues of $\G$  are also the eigenvalues of
$\T$. Let $\Gam_2=\diag(\lambda_{q{+}1}, \ldots, \lambda_{m{-}1})$.
We can express  the SVD of $\G$ as  $\G = [\Ph_q, \Ph_2] \left[\begin{array}{cc} \Gam_q & \0 \\
\0 & \Gam_2 \end{array} \right] \left[\begin{array}{c} \Ph_q^T \\
\Ph_2^T
\end{array} \right] = \Ph_q \Gam_q \Ph_q^T + \Ph_2 \Gam_2 \Ph_2^T$.
Then $\widehat{\Z} = \Ph_h (\Gam_q-\hat{\delta})^{\frac{1}{2}} \V^T$.
Substituting $(\widehat{\Z}, \hat{\delta})$ for $({\Z}, {\delta})$
yields $\widehat{\Z} \widehat{\Z}^T = \Ph_q (\Gam_q {-} \hat{\delta}
\I_q) \Ph_q^T$ and
\[
\widehat{\Tha}^{-1} = (\hat{\delta} \I_{m{-}1}+
\widehat{\Z}\widehat{\Z}^T)^{-1} = \hat{\delta}^{-1} \big[\I_{m{-}1}
- \Ph_q (\Gam_q- \hat{\delta} \I_q) \Gam_q^{-1}
\Ph_q^T\big]=\hat{\delta}^{-1}\Ph_2 \Ph_2^T + \Ph_q \Gam_q^{-1}
\Ph_q^T,
\]
which in turn lead to $\widehat{\Tha} -\G = \delta \Ps_2 \Ps_2^T -
\Ps_2 \Gam_2\Ps_2^T$, $\widehat{\Tha}^{-1} \G =
\hat{\delta}^{-1}\Ph_2 \Gam_2 \Ph_2^T + \Ph_q \Ph_q^T$,
$\widehat{\Tha}^{-2} = \hat{\delta}^{-2} \Ph_2 \Ph_2^T + \Ph_q
\Gam_q^{-2} \Ph_q^T$, $\widehat{\Tha}^{-3} \G = \hat{\delta}^{-3}
\Ph_2 \Gam_2 \Ph_2^T + \Ph_q \Gam_q^{-2} \Ph_q^T$ and
\[
\G \widehat{\Tha}^{-2} {+} \widehat{\Tha}^{-1} \G
\widehat{\Tha}^{-1} {-} \widehat{\Tha}^{-1} = \hat{\delta}^{-2}
\Ph_2  (2 \Gam_2 - \hat{\delta}\I_q) \Ph_2^T + \Ph_q \Gam_q^{-1}
\Ph_q^T.
\]
Let $\E_1= \Ph_h^T \X_0 \V$ and $\E_2= \Ph_2^T \X_0 \V$. It is then
obtained that
\[
B_1 \triangleq  \tr\big[\X_0 (\I_q {-} \widehat{\Z}^T
\widehat{\Tha}^{-1} \widehat{\Z}) \X_0^T (\widehat{\Tha} {-} \G)
\big] = \delta \tr\big[\E_2 \Gam_q^{-1} \E_2^T(\delta \I_{m{-}q{-}1}
{-} \Gam_2) \big],
\]
\[
B_2  \triangleq \tr\big[\X_0 \widehat{\Z}^T \widehat{\Tha}^{-1} \G
\widehat{\Tha}^{-1} \widehat{\Z} \X_0^T \widehat{\Tha} \big] =
\tr\big[\E_1 \Gam_q^{-1} (\Gam_q {-} \hat{\delta}\I_q) \E_1^T \Gam_q
\big] + \delta \tr\big[\E_2 \Gam_q^{-1} (\Gam_q {-} \delta\I_q) \E_2^T\big], 
\]
\[
B_3 \triangleq 2 \tr(\X_0 \widehat{\Z}^T \X_0 \widehat{\Z}^T
\widehat{\Tha}^{-1} \G) - \tr\big[\X_0 \widehat{\Z}^T \X_0
\widehat{\Z}^T\big]=  \tr\big[\E_1 (\Gam_q {-} \hat{\delta}
\I_q)^{\frac{1}{2}} \E_1  (\Gam_q {-} \hat{\delta}
\I_q)^{\frac{1}{2}} \big],
\]
\[
B_4 \triangleq 2 a \tr\big[\X_0^T (\G \widehat{\Tha}^{-2} {+}
\widehat{\Tha}^{-1} \G \widehat{\Tha}^{-1} {-} \widehat{\Tha}^{-1} )
\widehat{\Z} \big] = 2 a \tr\big[\E_1^T \Gam_q^{-1}(\Gam_q-
\hat{\delta}\I_q)^{\frac{1}{2}} \big],
\]
\begin{align*}
B_5 & \triangleq a^2 \tr(\widehat{\Tha}^{-3} \G) - \frac{a^2}{2}
\Big [\tr(\widehat{\Tha}^{-2}) + \hat{\delta}^{-2} \Big] +
\frac{a^2} {\hat{\delta}^{3}} \b^T \M \b =\frac{a^2}{2} \Big
[\tr(\widehat{\Tha}^{-2}) + \hat{\delta}^{-2} \Big]  \\
& = \frac{a^2}{2}\Big [\tr({\Gam_q}^{-2}) + (m{-}q)\hat{\delta}^{-2}
\Big].
\end{align*}
Thus,
\begin{align*}
 B & = B_1 + B_2 + B_3 + B_4 + B_5 \\
& =  \delta \left\{\tr\big[\E_2 \Gam_q^{-1} \E_2^T (\delta \I_{m{-}q{-}1}
{-} \Gam_2)
\big] + \tr\big[\E_2 \Gam_q^{-1} (\Gam_q {-} \delta\I_q)
\E_2^T\big]\right\} + \frac{m{-}q}{2}\hat{\delta}^{-2} a^2 \\
&  + \frac{1}{2}\tr\big [a \Gam_q^{-1} {+} 2\E_1  (\Gam_q {-} \hat{\delta}
\I_q)^{\frac{1}{2}} \big]^2  {-} \tr\big[\E_1 (\Gam_q {-} \hat{\delta} \I_q)^{\frac{1}{2}}
\E_1 (\Gam_q {-} \hat{\delta} \I_q)^{\frac{1}{2}} \big] {+}
\tr\big[\E_1 \Gam_q^{-1} (\Gam_q {-} \delta\I_q) \E_1^T \Gam_q \big].
\end{align*}
It is easily verified that $\tr\big[\E_2 \Gam_q^{-1} \E_2^T (\delta \I_{m{-}q{-}1}
{-} \Gam_2)
\big] + \tr\big[\E_2 \Gam_q^{-1} (\Gam_q {-} \delta\I_q)
\E_2^T\big] \geq 0$. On the other hand, let the $c_i + \imath d_i$ for $i=1, \ldots, q$ be the eigenvalues
of $a \Gam_q^{-1} {+} 2\E_1  (\Gam_q {-} \hat{\delta}
\I_q)^{\frac{1}{2}}$. It then follows from Lemma~\ref{lem:4} that
\[
\frac{1}{2}\tr\big [a \Gam_q^{-1} {+} 2\E_1  (\Gam_q {-} \hat{\delta}
\I_q)^{\frac{1}{2}} \big]^2 = \frac{1}{2} \sum_{i=1}^q  (c_i^2 - d_i^2).
\]
Furthermore, Lemma~\ref{lem:4}~(ii) shows that
\[
\frac{1}{2} \sum_{i=1}^q d_i^2 \leq  \tr\big[\E_1 \Gam_q^{-1} (\Gam_q {-} \delta\I_q) \E_1^T \Gam_q \big]{-} \tr\big[\E_1 (\Gam_q {-} \hat{\delta} \I_q)^{\frac{1}{2}}
\E_1 (\Gam_q {-} \hat{\delta} \I_q)^{\frac{1}{2}} \big].
\]
In summary, we prove that $B> 0$. This thus implies that $(\widehat{\A}, \hat{\delta})$ is the strict local minimizer of
$G_1(\A, \delta)$ under the constraint $\A^T\b=\0$.

Also, replacing $\BS$ for $\M$ in $G_1(\A, \delta)$, we immediately
obtain the strict local minimizer of $G(\A, \delta)$. In this case, since
$\BS^T\b=\0$, we have $\hat{\delta}=\frac{1}{m{-}q} \sum_{i=q+1}^{m}
\gamma_i$.

\section{The Proof of Lemma~\ref{lem:1}} \label{ap:a0}

We prove the lemma by induction on $t$. Let the rank of $\BS$ be $k$ ($\geq q$). Then we can write the condensed SVD of $\BS$
as $\BS=\B  \D \B^T$ where $\B$ is an $m{\times}k$ matrix with orthonormal columns and $\D$ is a $k{\times}k$ diagonal matrix with positive
diagonal entries.  Since $\ran(\A_{(0)}) \subseteq \ran(\BS)$, we are able to express $\A_{(0)}$ as $\A_{(0)}=\B \C$ where $\C$ is a $k{\times}q$
matrix of full-column rank. Subsequently, we have
\[
\Z_{(1)}= \BS \A_{(0)} =  \B \D \C,
\]
which implies the rank of $\Z_{(1)}$ is $q$. We now assume that $\A_{(t)}$ is of full-column rank.
In this case, the columns of $\Z_{(t{+}1)}=\BS \A_{(t)}$ are mutually independent. By induction,
we can derive $\A_{(t{+}1)}$ is a matrix of full-column rank.

\section{The Proof of Theorem~\ref{thm:2}} \label{ap:c}

We now prove that the $\delta$ computed by (\ref{eq:em2}) is
positive. Assume that we set the initial value of $\delta$ to a
positive number, i.e., $\delta_{(0)} >0$. Now supposing $\delta_{(t)}>0$,
we want to  prove that $\delta_{(t{+}1)} >0$. Substituting
(\ref{eq:em1}) into (\ref{eq:em2}), we have
\[
\delta_{(t{+}1)} = \frac{1}{m} \left[ \tr(\BS) - \tr\left(\BS {\A_{(t)}}
\big( \delta_{(t)}\I_q + \Si^{-1}_{(t)} \A_{(t)}^T  \BS {\A_{(t)}} \big)^{-1}
\Si^{-1}_{(t)} \A_{(t)}^T \BS \right) \right].
\]
Denote $\B = \BS- \BS \A  \big( \delta \Si + \A^T  \BS \A \big)^{-1}
\A^T \BS$. Eq.~(\ref{eq:em1}) shows that $\BS \BS^{+} {\A_{(t{+}1)}} =
\BS^{+} \BS {\A_{(t{+}1)}} = {\A_{(t{+}1)}}$ due to $\BS^{+} \BS
\BS=\BS$ and $\BS \BS^{+} \BS=\BS$. It is then easily proven that
$\B_{(t)} $ is the Moore-Penrose inverse of $\BS^{+} + \delta^{-1}_{t}
{\A_{(t)}}\Si^{-1}_{t} \A_{(t)}^T$ \citep{Harville:1977}. As a result,
$\B_{(t)} $ is p.s.d.\ due to positive semidefiniteness of $\BS$ and
$\A\Si^{-1} \A^T$. Thus, $\tr(\B_{(t)})$ is positive.

It is well known that the standard EM algorithm converges to a local
minimum or a  saddle point. In ay case,  assume ${\A_{(t)}} \rightarrow \widehat{\A}$ and
$\delta_{(t)} \rightarrow \hat{\delta}$. It follows from
(\ref{eq:em1}) and (\ref{eq:em2}) that
\[
\widehat{\A} = \BS \widehat{\A} \left(\hat{\delta}\I_q +
\widehat{\Si}^{-1} \widehat{\A}^T \BS \widehat{\A} \right)^{-1}
\]
\[
\hat{\delta} = \frac{1}{m} \left[\tr(\BS) - \tr(\widehat{\A}
\widehat{\Si}^{-1} \widehat{\A}^T \BS) \right]
\]
We thus have $ \widehat{\A} (\hat{\delta}\I_q + \widehat{\Si}^{-1}
\widehat{\A}^T \BS \widehat{\A} )= \BS \widehat{\A}$. Since
$\widehat{\A} \widehat{\Si}^{-1} = \widehat{\A} (\hat{\delta} \I_q +
\widehat{\A}^T\widehat{\A})^{-1} = (\hat{\delta} \I_m + \widehat{\A}
\widehat{\A}^T)^{-1} \widehat{\A}$, we obtain $\BS \widehat{\A} =
\widehat{\A}(\hat{\delta} \I_q + \widehat{\A}^T\widehat{\A})$. Let
$\V \Lam \V^T =\widehat{\A}^T \widehat{\A}$ be SVD of $\widehat{\A}^T \widehat{\A}$. Then $\BS
\widehat{\A} \V \Lam^{-\frac{1}{2}} = \widehat{\A} \V
\Lam^{-\frac{1}{2}}(\hat{\delta} \I_q + \Lam)$. This implies that
$\hat{\delta} \I_q + \Lam$ and $\widehat{\A} \V \Lam^{-\frac{1}{2}}$
are the eigenvalue matrix and corresponding eigenvector matrix of
$\BS$. According to Appendix~\ref{ap:a}, we have $\widehat{\A} =
{\U_q}({\Gam_q} - \hat{\delta} \I_q)^{\frac{1}{2}} \V^T$. In this
case, because of $\tr(\widehat{\A} \widehat{\Si}^{-1} \widehat{\A}^T
\BS)= \tr({\Gam_q}) - q \hat{\delta}$, we have $\hat{\delta}
=\frac{1}{m{-}q} \sum_{j=q+1}^{m} \gamma_j$.

\section{Derivation of the EM Algorithm} \label{ap:ff}

In the case that $\u=\frac{1}{\1_m \b} \F^T \b$, we have $\F-\1_m \u^T
= \H_b \F$. It is readily seen that
\[
\H_b \F | \W \thicksim N_{m, r}\left(\A \W , \; \delta
(\I_m{\otimes} \I_r)/r \right).
\]
Using Bayes' rule, we can compute the conditional distribution of
$\W$ given $\H_b \F$ as
\begin{equation} \label{eq:condition}
\W |\H_b \F \thicksim N_{q, r} \big( \Si^{-1} \A^T \H_b \F, \;
\delta(\Si^{-1} \otimes  \I_r)/r \big),
\end{equation}
where $\Si = \delta\I_q +  \A^T \A$.

Considering $\W$ as the missing data, $\{\W, \H_b \F\}$ as the
complete data, and $\A$ and $\delta$ as the model parameters, we now
devise an EM algorithm for the ridge approximation. First, the
complete-data log-likelihood is
\begin{eqnarray*}
L_c &  = & \log p(\W, \; \H_b \F)  =  \log p(\H_b \F \mid \W) + \log p(\W) \\
& \varpropto & - \frac{m r}{2}\log \delta - \frac{r}{2} \tr \left(
\W \W^T \right) - \frac{r }{2 \delta}  \tr\left((\H_b \F -\A\W) (\H_b
\F -\A\W)^T \right),
\end{eqnarray*}
where we have omitted the terms independent of $\A$ and $\delta$. It
is easy to find that $\W$ and $\W \W^T$ are the complete-data
sufficient statistics for $\A$ and $\delta$.

Using some properties of matrix-variate normal distributions
\citep[Page 60]{GuptaN:Book:2000}, we have
\begin{eqnarray}
\EB(\W |\H_b \F) &=&  \Si^{-1} \A^T \H_b \F, \label{eq:exp} \\
\EB(\W \W^T |\H_b \F) &=& \delta \Si^{-1} + \Si^{-1} \A^T \BS \A
\Si^{-1}. \label{eq:var}
\end{eqnarray}

Given the $t$th estimates, ${\A_{(t)}}$ and $\delta_{(t)}$, of $\A$ and
$\delta$, the E-step computes the expectation of $L_c$ w.r.t.\ $p(\W
|\H_b\F, {\A_{(t)}}, \delta_{(t)})$, namely,
\begin{eqnarray*}
Q(\A, \delta |\H_b \A_{(t)}, \delta_{(t)}) & = &-\frac{m r}{2} \log
\delta  - \frac{r}{2} \tr \left( \langle \W \W^T \rangle \right)
- \frac{r}{2\delta}  \tr\left( \BS \right) \\
& & - \frac{r}{2\delta} \tr\left(\A \langle\W \W^T \rangle \A^T
\right) + \frac{r}{\delta} \tr\left(\A \langle \W \rangle \F^T
\H_b^T\right),
\end{eqnarray*}
where $\langle \W \rangle = \EB(\W |\H_b \F, {\A_{(t)}}, \delta_{(t)})$
and $\langle \W \W^T \rangle = \EB(\W \W^T|\H_b\F, {\A_{(t)}},
\delta_{(t)})$. It follows from (\ref{eq:exp}) and (\ref{eq:var}) that
\begin{eqnarray}
\langle \W \rangle &=&  \Si^{-1}_{(t)} \A_{(t)}^T \H_b \F, \label{eq:suff1} \\
\langle \W \W^T \rangle &=& \delta_{(t)} \Si^{-1}_{(t)} + \Si^{-1}_{(t)}
\A_{(t)}^T \BS  \A_{(t)} \Si^{-1}_{(t)}. \label{eq:suff2}
\end{eqnarray}
The M-step maximizes $Q(\A, \delta | {\A_{(t)}}, \delta_{(t)})$ w.r.t.\
$\A$ and $\delta$, giving their new estimates:
\begin{eqnarray}
{\A_{(t{+}1)}} &=&\H_b \F \langle \W^T \rangle  \left( \langle \W \W^T
\rangle \right)^{-1},
\label{eq:m1} \\
\delta_{(t{+}1)} &=& \frac{1}{m} \Big[ \tr(\BS) {+}
\tr\left(\A_{(t{+}1)}^T \langle \W \W^T \rangle {\A_{(t{+}1)}} {-} 2
\A^T_{(t{+}1)} \H_b \F \langle \W^T \rangle  \right) \Big].
\label{eq:m2}
\end{eqnarray}
It then follows from (\ref{eq:m1}) that
\[
 {\A_{(t{+}1)}} \langle \W \W^T \rangle =  \H_b \F \langle \W^T \rangle.
\]
Thus, we can rewrite (\ref{eq:m2}) as
\begin{equation} \label{eq:m22}
\delta_{(t{+}1)} = \frac{1}{m} \Big[ \tr(\BS)  - \tr(\A_{(t{+}1)}^T \H_b
\F \langle \W^T \rangle )  \Big].
\end{equation}

Now substituting $\langle \W \rangle $ and $\langle \W \W^T \rangle $
from (\ref{eq:suff1}) and (\ref{eq:suff2}) into (\ref{eq:m1}) and
(\ref{eq:m22}), we can combine the E-step and M-step into
(\ref{eq:em1}) and (\ref{eq:em2}).

\bibliographystyle{Chicago}
\bibliography{mra}


\end{document}